\documentclass{l4dc2024}
\PassOptionsToPackage{ruled,noend}{algorithm2e}

\PassOptionsToPackage{numbers, compress}{natbib}
\usepackage[utf8]{inputenc} 
\usepackage[T1]{fontenc}    
\usepackage{hyperref}       
\usepackage{url}            
\usepackage{booktabs}       
\usepackage{amsfonts}       
\usepackage{nicefrac}       
\usepackage{microtype}      
\usepackage{xcolor}         
\usepackage{colortbl}
\definecolor{Gray}{gray}{0.9}
\usepackage{makecell}

\usepackage{multirow}
\usepackage{caption}

\usepackage{graphicx}
\usepackage{flafter}
\usepackage[shortlabels]{enumitem}
\usepackage{bbm}    
\usepackage{bm}
\usepackage{dsfont}
\usepackage{array,multirow}
\usepackage{float}
\usepackage{wrapfig}
\usepackage{hyperref}
\usepackage{tcolorbox}
\usepackage{mathrsfs}
\usepackage{stmaryrd}
\usepackage{placeins}

\usepackage{amsmath}
\usepackage{amssymb}
\usepackage{pifont}
\newcommand{\cmark}{\ding{51}}%
\newcommand{\xmark}{\ding{55}}%
\usepackage{mathtools}
\usepackage{gensymb}

\usepackage{cleveref}
\DeclareMathOperator*{\argmin}{argmin}
\DeclareMathOperator*{\argmax}{argmax}

\newtheorem*{theorem*}{Theorem}
\newtheorem*{proposition*}{Proposition}
\newtheorem*{assumption*}{Assumption}

\newenvironment{subtheorems}
 {\itemize[
   nosep,font=\bfseries,
   leftmargin=3em,itemindent=-1em,align=left]}
 {\enditemize}

\renewcommand{\thefootnote}{\arabic{footnote}}

\newcommand{\bmtheta}[0]{{\bm{\theta}}}
\newcommand{\bmeps}[0]{{\bm{\epsilon}}}
\newcommand{\bmeta}[0]{{\bm{\eta}}}
\newcommand{\bmf}[0]{{\bm{f}}}
\newcommand{\bmx}[0]{{\bm{x}}}
\newcommand{\bms}[0]{{\bm{s}}}
\newcommand{\bmz}[0]{{\bm{z}}}
\newcommand{\bmq}[0]{{\bm{q}}}

\newcommand{\Vy}[0]{{V_y}}

\newcommand{\eps}[0]{{\bar{\epsilon}}}

\newcommand{\ours}[0]{FI-ODE }

\newcommand{\BR}{\mathbb R}

\newcommand{\den}{{N}}
\newcommand{\state}{{\bm{\eta}}}
\newcommand{\params}{{\bm{\theta}}}
\newcommand{\sparams}{{\bm{\theta}^*}}
\newcommand{\dyn}{{\bm{f}}}

\newcommand{\ins}{{\bm{x}}}

\newcommand{\adv}{{\bm{\epsilon}}}
\newcommand{\advbnd}{{\overline{\epsilon}}}

\newcommand\blfootnote[1]{%
  \begingroup
  \renewcommand*{\footnoteseptext}{}
  \renewcommand\thefootnote{}\footnotetext{#1}%
  \endgroup
}

\title[FI-ODE]{FI-ODE: Certifiably Robust Forward Invariance in Neural ODEs}
\usepackage{times}

\author{%
 \Name{Yujia Huang} $^{*\dagger}$ \Email{yjhuang@caltech.edu} \\
 \Name{Ivan Dario Jimenez Rodriguez} $^{*\dagger}$ \Email{ivan.jimenez@caltech.edu} \\
 \Name{Huan Zhang}$^{\ddag}$ \Email{huan@huan-zhang.com} \\
 \Name{Yuanyuan Shi}$^{\dagger\dagger}$ \Email{yus047@ucsd.edu} \\
 \Name{Yisong Yue} $^{\dagger}$ \Email{yyue@caltech.edu}  \\
 \addr $^\dagger$California Institute of Technology \quad $^{\ddag}$University of Illinois Urbana-Champaign \quad $^{\dagger\dagger}$ UC San Diego  %
}

\begin{document}

\maketitle
\vspace{-35pt}
\begin{abstract}%
Forward invariance is a long-studied property in control theory that is used to certify that a dynamical system stays within some pre-specified set of states for all time, and also admits robustness guarantees (e.g., the certificate holds under perturbations).
We propose a general framework for training and provably certifying robust forward invariance in Neural ODEs.
We apply this framework to provide certified safety in robust continuous control. 
To our knowledge, this is the first instance of training Neural ODE policies with such non-vacuous certified guarantees.
In addition, we explore the generality of our framework by using it to certify adversarial robustness for image classification. 
\end{abstract}


\blfootnote{$^*$ These authors contributed equally to this work.}
\section{Introduction}
We study the problem of training neural networks with certifiable performance guarantees.  
Example performance criteria include safety in control \citep{jin2020neural}, and  adversarial robustness in classification  \citep{wong2018provable,raghunathan2018certified,cohen2019certified}, where even impressive empirical robustness often fails under unforeseen stronger attacks \citep{athalye2018obfuscated}.
As such, having formal performance certificates can be valuable when deploying neural networks in high-stakes real-world settings.

In this paper, we are interested in performance criteria characterized by a property called forward invariance.
Forward invariance has been extensively used in control theory to certify dynamical systems for safety \citep{ames2016control} and robustness under adversarial perturbations \citep{khalil1996robust}.
To use this concept for machine learning, we focus on the Neural ODE (NODE) function class \citep{haber_stable_2017,e_proposal_2017,chen2018neural}, which is a natural starting point for incorporating control-theoretic tools (cf. \cite{yan2019robustness,kang2021stable,liu2020does,jimenez2022lyanet}).
Forward invariance guarantees that NODE trajectories never leave a specified set, which can be translated into various robust safety guarantees.
Given the increasing interest in using NODE policies for robotic control \citep{bottcher2022near,lin2021no}, including those that are forward invariant \citep{rodriguez2022neural}, having certified NODE controllers will become important as those methods are more widely adopted.

\textbf{Our contributions:}
We present FI-ODE, a general approach for training certifiably robust forward invariant NODEs.\footnote{Note that training certifiably forward invariant NODEs even in the non-robust setting is itself a contribution.}
Our approach is based on defining forward invariance using sub-level sets of Lyapunov functions. 
One can train a NODE such that a task-specific cost function (e.g., state-based cost in continuous control, or cross-entropy loss in image classification) becomes the Lyapunov function for the ODE. 
We train with an adaptation of Lyapunov training \citep{jimenez2022lyanet} that focuses on states that are crucial for certifying robust forward invariance.
To make certification practical, we constrain the hidden states of a NODE to evolve on a compact set by projecting the dynamics of NODE to satisfy certain barrier conditions.
We provably verify our method through a combination of efficient sampling and a new interval propagation technique compatible with optimization layers.

We evaluate using a canonical unstable nonlinear system (planar segway).
We demonstrate certified robust forward invariance of the induced region of attraction, which to our knowledge is the first NODE policy with such non-vacuous certified guarantees.
To show generality, we also evaluate on image classification, and show superior $\ell_2$ certified robustness versus other certifiably robust ODE-based models. 
Our code is available at \url{https://github.com/yjhuangcd/FI-ODE.git}.

\section{Preliminaries}
\label{sec:background}

\paragraph{Neural ODEs.}
We consider the following Neural ODE (NODE) model class, where $\bm{x}$ are the inputs to the dynamics, and $\bm{\eta} \in \mathcal{H} \subset \BR^n$ are the states of the NODE ($\mathcal{H}$ is compact and connected). Let  $\bm{\theta} \in \Theta \subseteq \BR^l$ denote the parameters of the learned model. In general, we assume the overparameterized setting, where $\bm{\theta}$ is expressive enough to fit the dynamics.

\begin{subequations}
\begin{align}
    \label{eq:neuralode_input} \bm{\eta}(0) &= \bm{\eta}_0,  \ \ \ \ \ \ \ \ \ \ \ \ \ \ \  \mbox{(initial condition)} \\
    \label{eq:neuralode_hidden} \frac{d\bm{\eta}}{dt} &= \bm{f}_{\bm{\theta}}(\bm{\eta}(t), \bm{x}) \ \ \mbox{(continuum of hidden layers)}.
\end{align}
\end{subequations}
An important setting for NODEs is continuous control, where we can explicitly compose the known dynamics of the physical system with a neural network controller parameterized by $\bm{\theta}$. The closed-loop system is denoted by $\bm{f}_{\bm{\theta}}(\bm{\eta}(t), \bm{x})$, where $\bm{\theta}$ is the neural controller and $\bm{x}$ are the system parameters. Other settings include image classification, where $\bm{x}$ are the images, and we evolve the system over $t \in [0,T]$ to get the final prediction $\bm{\eta}(T)$.\footnote{In this setting, one can think of a NODE as a ResNet \citep{he2016deep} with a continuum of hidden layers.}

\paragraph{Forward Invariance \& Robust Forward Invariance.}
Forward Invariance refers to sets of states of a dynamical system (e.g., \Cref{eq:neuralode_hidden}) where the system can enter but never leave. 
Formally:
\begin{definition}[Forward Invariance]
\label{def:FI}
    A set $\mathcal{S} \subseteq \mathcal{H}$ is forward invariant with respect to the system (\Cref{eq:neuralode_hidden}) if $\bm{\eta}(t) \in \mathcal{S} \Rightarrow \bm{\eta}(t') \in \mathcal{S}, \forall t'\geq t$.
\end{definition}
Forward invariance can be applied generally in NODEs: we can choose the dynamics in \Cref{eq:neuralode_hidden} to render almost any set we choose forward invariant.
For instance, in control we often want to keep the states of the system within a safe set, while in classification we will be concerned with the set of states that produce a correct classification.
Mathematically, these settings can be captured by shaping the ODE dynamics $\bm{f}_{\bm{\theta}}$ to achieve forward invariance within a specified set (i.e., training the ODE to satisfy Definition \ref{def:FI} for some specified set $\mathcal{S}$).

\begin{definition}[Robust Forward Invariance]
    A set $\mathcal{S} \subseteq \mathcal{H}$ is robust forward invariant with respect to $\bm{x}$ if $\mathcal{S}$ is forward invariant with respect to the system $\bm{f}_{\bm{\theta}}(\bm{\eta}(t), \bm{x} + \adv)$, $\forall \adv \in \BR^n$ with $\lVert \adv \rVert \leq \advbnd$. 
\end{definition}

\begin{figure}[t]
  \begin{minipage}[c]{0.45\linewidth}
\centering
\includegraphics[width=\linewidth]{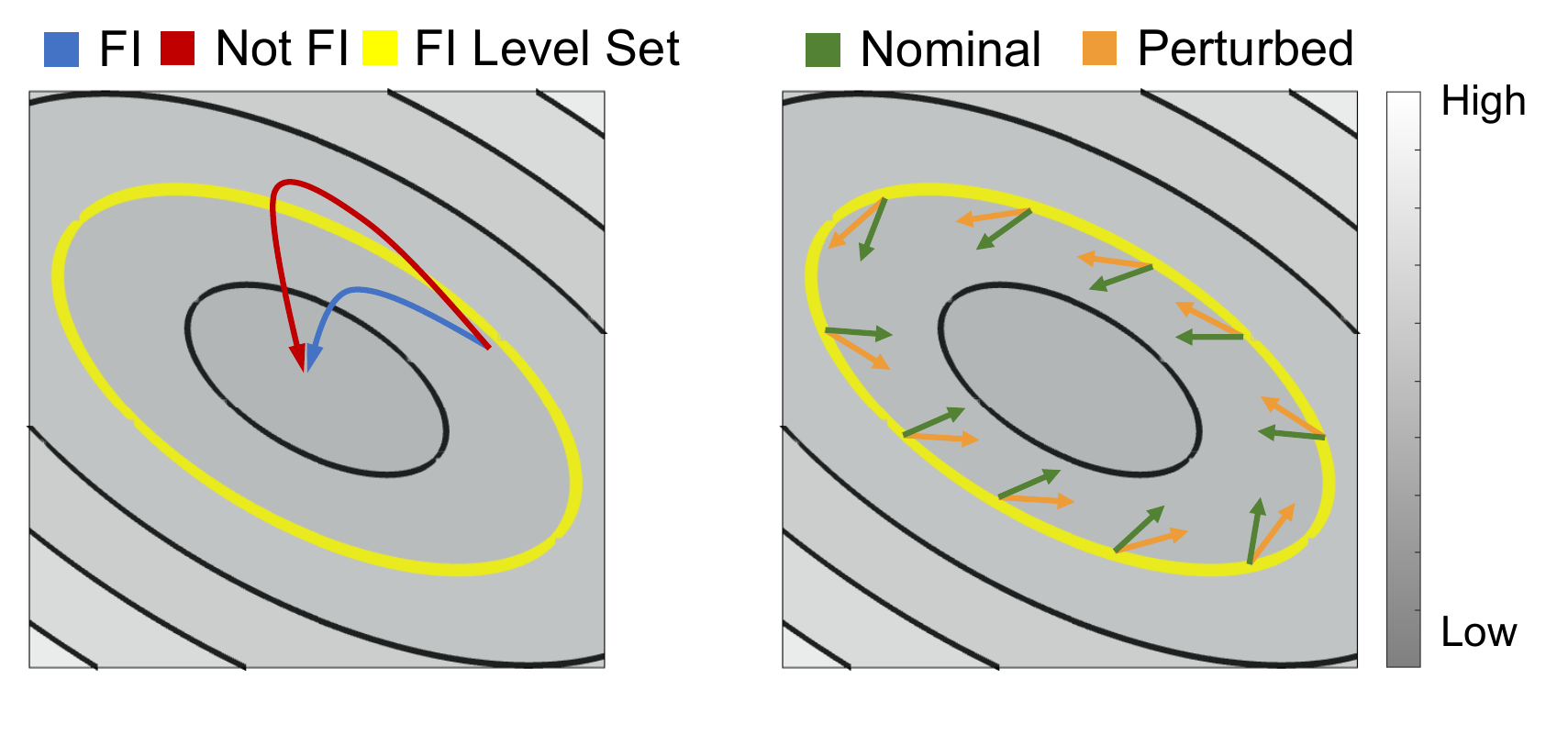}
     \end{minipage}
  \hfill
  \begin{minipage}[c]{0.55\linewidth}
  \caption{
  Depicting trajectories (Left) and dynamics (Right) of the state-space of a NODE. 
  The contours show a quadratic potential, and the yellow-line is the target sublevel set.
  Left:  trajectories that violate (red) or satisfy (blue)  forward invariance.
  Right: flow field (dynamics) of the NODE, under both nominal and perturbed inputs. The perturbed flow field still satisfies forward invariance, implying robust forward invariance. 
  }
  \label{fig:fi-traj}
    \end{minipage}
    \vspace{-10pt}
\end{figure}

Robust forward invariance is attractive when one seeks performance guarantees under input perturbations. Here, we consider norm-bounded perturbations. In control, when there are mis-specifications for system parameters, we still hope the controller to be able to keep the system safe. In classification, we would want the system to classify correctly despite noisy inputs.

\paragraph{Trajectory-wise versus Point-wise Certification Analysis.}
\Cref{fig:fi-traj} depicts two ways of certifying forward invariance. 
On the left, we consider entire trajectories that result from running the ODE (i.e., running the forward pass) and determine forward invariance by checking whether the \emph{trajectories} leave the target set.
Such trajectory-level analyses are computationally expensive due to running ODE integration to generate  trajectories.
This approach also poses a challenge for verification since trajectories can only be integrated for finite time $T$ and the dynamics may be close to leaving the set shortly thereafter ($T+\varepsilon$), which translates into vulnerability to perturbations.

An alternative approach, depicted on \Cref{fig:fi-traj}(right), relies on point-wise conditions: we look at the dynamics \emph{point-wise} over the state-space and infer whether the set within the yellow line is forward invariant.
Here, robust certification can be significantly easier because we only need to verify that the perturbed dynamics are point-wise still pointing in the right direction, rather than analyzing the perturbed dynamics over an entire trajectory (i.e., we do not need to do ODE integration).

\begin{figure*}[t]
\centering
\includegraphics[width=0.95\textwidth,trim={0 0 0 0},clip]{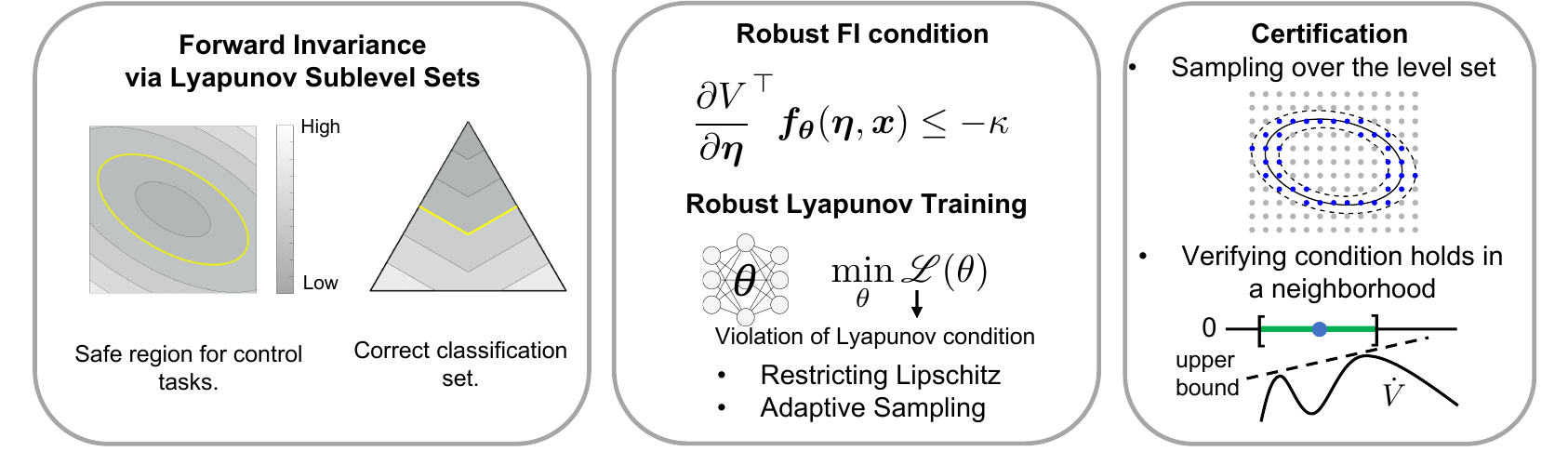}
\vspace{-5pt}
\caption{Overview of our FI-ODE framework. We first pick a Lyapunov function based on the shape of the forward invariant set: the boundaries of the Lyapunov sublevel sets are parallel to the boundary of the forward invariant set. Then we show that robust forward invariance implies robust control and classification. We train the dynamics to satisfy robust FI conditions via robust Lyapunov training. To certify the forward invariance property, we sample points on the boundary of the forward invariant set and verify conditions hold everywhere on the boundary. }
\label{fig:overview}
\end{figure*}

\paragraph{Lyapunov Functions \& Sublevel Sets.} As discussed further in \Cref{sec: FI-sublevel-set}, we use Lyapunov potential functions from control theory to define sets to render forward invariant. A potential function 
$V: \mathcal{H} \to \BR_{\geq 0}$  is a Lyapunov function for the ODE if for all reachable states $\eta$ we have:
\begin{equation}
\label{eq: clean_condition}
    \dot{V} \equiv \frac{d}{dt} V(\bm{\eta}(t)) \equiv \frac{\partial V}{\partial \bm{\eta}}^\top \bm{f_\params}(\bm{\eta}, \bm{x}) \leq 0.
\end{equation}
Intuitively, Equation \eqref{eq: clean_condition} means that the dynamics of the ODE within the states $\mathcal{H}$ are always flowing in the direction that reduces $V$, i.e., Equation \eqref{eq: clean_condition} establishes a contraction condition on ODE. Note that for the system to be strictly contracting, the RHS of Equation \eqref{eq: clean_condition} needs to be strictly negative.

Since $V$ is always decreasing in time, we can use it to define a forward invariant set (Definition \ref{def:FI}): $V(\state)\leq c$ for some constant $c$, which is known as a Lyapunov sublevel set.
Once a state enters a Lyapunov sublevel set it remains there for all time.
The  potential function depicted in \Cref{fig:fi-traj}  can be viewed as a Lyapunov function, and the yellow line the boundary of the corresponding sublevel set. At training time, one would specify a desired forward invariance condition using a potential function $V$ and threshold $c$, and optimize the NODE to satisfy the forward invariance condition.

\section{FI-ODE: Robust Forward Invariance for Neural ODEs}
We now present our FI-ODE framework to enforce forward invariance on NODEs (\Cref{fig:overview}). 
We define forward invariance using Lyapunov sublevel sets (\Cref{sec: FI-sublevel-set}), and show that robust forward invariance implies robust control and classification (\Cref{sec: robust-FI-implication})
To enforce forward invariance, we first train to encourage  the Lyapunov conditions to hold on the boundary of the target Lyapunov sublevel set (e.g., yellow line in \Cref{fig:fi-traj}), and then verify. 
We develop a robust Lyapunov training algorithm  (\Cref{sec:learning}) that extends the LyaNet framework \citep{jimenez2022lyanet} to enable efficiently training NODEs that provably satisfy forward invariance.
Finally, we develop certification tools to verify the Lyapunov conditions everywhere in the region of interest (\Cref{sec: verification}).

\subsection{Forward Invariance via Lyapunov Sublevel Sets}
\label{sec: FI-sublevel-set}

We first define the set that we would like to render forward invariant, and then choose a Lyapunov function whose level sets are parallel to the boundary of the set.\footnote{Usually, Lyapunov stability uses a potential function to prove the stability of a given dynamical system. In our setting, the potential function is pre-defined to be positive definite, and we find a dynamical system (e.g. by training a Neural ODE) that is stable with respect to this potential function (i.e. making this potential function a Lyapunov function). 
This is possible because the NODEs are typically overparameterized.}
For our main application in control, we define the forward invariant set to be a region around an equilibrium point (i.e., straying far from the equilibrium point can be unsafe), and use the standard quadratic Lyapunov function:
\begin{equation}
\label{eq: quadratic_lya}
    V(\bmeta) = \bmeta^\top P \bmeta,
\end{equation}
where $P$ is a (learnable) positive definite matrix, and assuming WLOG that the equilibrium point is at the origin.
The forward invariant set has the form $\mathcal{S} = \{ \bmeta | V(\bmeta) \leq c \}$, for $c>0$. The level sets of this quadratic Lyapunov function are shown in \Cref{fig:fi-traj}. The boundary is then $\mathcal{D} = \partial \mathcal{S} = \{ \bmeta | V(\bmeta) = c \}$.  This forward invariance condition is commonly used in safety-critical control \citep{ames2019control}.

We also explore an application to multi-class classification.  Here, we define the forward invariant set to be the correct classification region. For an input $\bmx$ with label $y$, the output of a  NODE after integrating for $T$ time is $\bmeta(T)$. The NODE correctly classifies $\bmx$ if $y = \argmax \bmeta(T)$. Then the correct classification region for class $y$ is $\mathcal{S}_y = \{ \bmeta | \bmeta \in \triangle, y = \argmax \bmeta \}$ (\Cref{fig:overview}, left panel) 
where $\triangle$ stands for the $n$-class probability simplex: $ \{ \bmeta \in \mathbbm{R}^n | \sum_{i=1}^n \bmeta_i=1, \bmeta_i \geq 0 \}$. 
The boundary of this set is known to be the decision boundary for class $y$: $\mathcal{D}_y = \{ \state \in \triangle | \bmeta_y = \max_{i \neq y} \bmeta_i \}$.
We define a Lyapunov function whose level sets are parallel to the decision boundary:
\begin{equation}
\label{eq: margin_lya}
    \Vy(\bmeta) = 1 - (\bmeta_y - \max_{i \neq y} \bmeta_i )
\end{equation}
We can check that $V_y$ is positive definite: since $0 \leq \eta_i \leq 1$ for all $i$, we have $V_y \geq 0$. In addition, $V_y = 0$ only when $\eta_y = 1$ and $\eta_i = 0$ for $i \neq y$. For the simplicity of notations, we  use $V$ to refer to the Lyapunov function, but note that the Lyapunov function for classification depends on class $y$.

\subsection{Robust Forward Invariance for Robust Control and Classification}
\label{sec: robust-FI-implication}
A NODE satisfies robust forward invariance if the forward invariance condition holds despite (norm-bounded) perturbations on the dynamics (e.g. due to perturbed inputs).
Our framework uses $\bmx$ for \emph{system parameters} in control, and for \emph{input images} in classification. 
To ensure robust forward invariance for perturbed $\bmx$, the dynamics for the perturbed input $\bm{f_\params}(\bm{\eta}, \bmx+\bmeps)$ needs to satisfy the standard forward invariance condition in \eqref{eq: clean_condition} (as informally depicted in \Cref{fig:fi-traj}, right). 
In other words, the condition in \eqref{eq: clean_condition} needs to hold in a neighborhood of $\bmx$ for robust control or classification.
Thanks to the Lipschitz continuity of the Lyapunov function $V$ and the dynamics $\bm{f_\params}$, this can be achieved by a more strict condition than \eqref{eq: clean_condition} on the dynamics (\Cref{thm:robust_FI}).

\begin{theorem}[Robust Forward Invariance]
\label{thm:robust_FI}
Consider the dynamical system in \Cref{eq:neuralode_input,eq:neuralode_hidden}, the set $\mathcal{S}$ will be robust forward invariant with respect to $\bmx$ if the following conditions hold:
\begin{align}
    \frac{\partial V}{\partial \bm{\eta}}^\top \bm{f_\params}(\bm{\eta}, \bm{x}) & \leq - \eps L_{V} L_{f}^x, \quad \forall \bmeta \in \partial \mathcal{S} \label{eq:robust_fi}.
\end{align}
where $\eps$ is the perturbation magnitude on $\bmx$ (i.e. $\lVert \adv \rVert \leq \advbnd$), $\partial \mathcal{S}$ is the boundary of $\mathcal{S}$, $L_{V}$ is the Lipschitz constant of $V$ and $L_f^x$ is the Lipschitz constant of the dynamics with respect to $x$.
\end{theorem}

\begin{remark}[Implications]
    With $\mathcal{S}$ defined as in \Cref{sec: FI-sublevel-set}, if the dynamics satisfy \eqref{eq:robust_fi}, then we have a robust controller that always keeps the system in the desired region despite perturbed system parameters and inputs. 
\end{remark}

\begin{remark}[Non-Robust Variant]
    The non-robust version of \Cref{thm:robust_FI} is where the RHS of \Cref{eq:robust_fi} is 0 instead of $-\eps L_{V} L_{f}^x$.  I.e., \Cref{eq:robust_fi} need not be a strictly contracting condition.
\end{remark}

\subsection{Robust Lyapunov Training}
\label{sec:learning}
We now present our robust Lyapunov training approach to satisfy the conditions in Theorem \ref{thm:robust_FI} (\Cref{alg:training}).
Our method extends the LyaNet framework \citep{jimenez2022lyanet} in two ways: 1) restricting the Lipschitz constant of the NODE with respect to the input; and 2) adaptive sampling to focus learning on the states necessary for forward invariance certification. 

\begin{algorithm2e}[t]
\SetAlCapHSkip{0.0em}
\SetAlgoLined
\SetKwInOut{Input}{Input}\SetKwInOut{Initialize}{Initialize}
\Input{Lyapunov function $V$, Sampling scheduler, dataset $\mathcal{D}$, hinge-like function $\kappa$.}
\Initialize{Model parameters $\bmtheta$, Lyapunov parameters $P$ and system parameters $\ins$ (for control).}
\SetInd{.1em}{0.3em}
\For{$i=1:M$}{
 $\triangleright$ \small{Sample $\bmeta$ based on the training progress and the level sets of $V$:} $\bmeta \sim$  \texttt{Sampling\_scheduler} ($i, V$) \\
 $\triangleright$ For control, find adversarial samples of $\ins$ and $\state$ \\
 $\;\;$ For classification, sample $(\bmx,y) \sim \mathcal{D}$ \\
 $\triangleright$ \small{Update model parameters to minimize Lyapunov loss $\mathscr{L}(\params)$ (\Cref{eq:mc_lya_loss}).} \\
 $P \leftarrow P - \beta' \nabla_{P} \mathscr{L}(\params)$ \\ 
 $\bmtheta \leftarrow \bmtheta - \beta \nabla_{\bmtheta} \mathscr{L}(\params)$ 
}
\Return{$\bmtheta$}
\caption{Robust Lyapunov Training}
\label{alg:training}
\end{algorithm2e}

\paragraph{Training loss.}
Our training loss encourages the dynamics to satisfy the conditions in \Cref{thm:robust_FI}. Specifically, we use a modified Monte Carlo Lyapunov loss from \citep{jimenez2022lyanet}:
\begin{align}
     \mathscr{L}(\params) \approx \underset{ \state \sim \mu\left(\mathcal{H}\right)} {\mathbbm{E}} \left[ \max\left\{0,   \frac{\partial V}{\partial \state }^\top \dyn_\params(\state, \ins) + \kappa(V(\bmeta)) \right\} \right], \label{eq:mc_lya_loss}
\end{align}
which can be interpreted as a hinge-like loss on the Lyapunov contraction condition for each state $\eta$ of the NODE (i.e., the loss encourages $(\partial V/\partial \bmeta)^\top\dyn \leq -\kappa(V(\bmeta)) < 0$ for some non-negative non-decreasing function $\kappa$). Intuitively, if the loss in \Cref{eq:mc_lya_loss} is 0 for some given $\bmx$, then we know that the contraction condition is satisfied with the RHS being $\kappa(V(\bmeta))$. As long as $\kappa(V(\bmeta))\geq \eps L_{V} L_{f}^x$, then \Cref{thm:robust_FI} holds.  If instead $\kappa(V(\bmeta))\geq 0$, then the non-robust variant holds.

\paragraph{Restricting the Lipschitz constant.}
To obtain a non-vacuous guarantee from \Cref{thm:robust_FI}, we need to restrict the Lipschitz of $\dyn_\params(\state, \ins)$ with respect to both $\state$ and $\ins$. For image classification, we can estimate $L_f^x$ easily by the product of matrix norms of the weight matrices in the neural network \citep{tsuzuku2018lipschitz}. Then we certify condition \Cref{eq:robust_fi} holds for the clean image $\ins$.
For control problems, since the dynamics of the physical system is usually known, the closed loop dynamics is not purely parameterized by neural networks and it is not straightforward to estimate $L_f^x$. Therefore, instead of directly certifying condition \Cref{eq:robust_fi}, we certify the LHS of it to be smaller than 0 for $\bmeta \in \partial \mathcal{S}$ and all $\ins$ within the perturbation range (not only on the nominal parameter). We use adversarial training to make $\dyn_\params(\state, \ins)$ smooth with respect to both $\state$ and $\ins$, and in fact certifiable.

\paragraph{Adaptive sampling.}
To minimize the Lyapunov loss (\Cref{eq:mc_lya_loss}), we need to choose a sampling distribution $\mu$. 
A simple choice is uniform (as was done in \citep{jimenez2022lyanet}), but that may require an intractable number of samples to guarantee minimizing~\Cref{eq:mc_lya_loss} everywhere in the state space. 
We address this challenge with an adaptive sampling strategy that focuses training samples on the region of the state space necessary for forward invariance: the boundary of the Lyapunov sub-level set.
For control problems, since we jointly learn matrix $P$ in the Lyapunov function, the shape of its level set changes during training, and the sampled points change accordingly. 
For classification, we switch from uniform sampling in the simplex to sampling only within the forward invariant set, with the switching time being a hyper-parameter of the sampling scheduler. 

\subsection{Certification}
\label{sec: verification}

In the previous section, we minimize the empirical Lyapunov loss (\Cref{eq:mc_lya_loss}) on some finite set of samples to encourage the dynamics to satisfy conditions in \Cref{thm:robust_FI}. However, zero empirical Lyapunov loss on a finite sample is not necessarily a certificate that the conditions hold everywhere on the boundary of the forward invariant. This section develops tools to certify the forward invariance conditions hold \emph{everywhere} on the boundary of the safe set.

\begin{wrapfigure}{R}{0.4\textwidth}
\vspace{-12pt}
  \centering
    \includegraphics[width=1\linewidth, trim={0 0 0 0},clip]{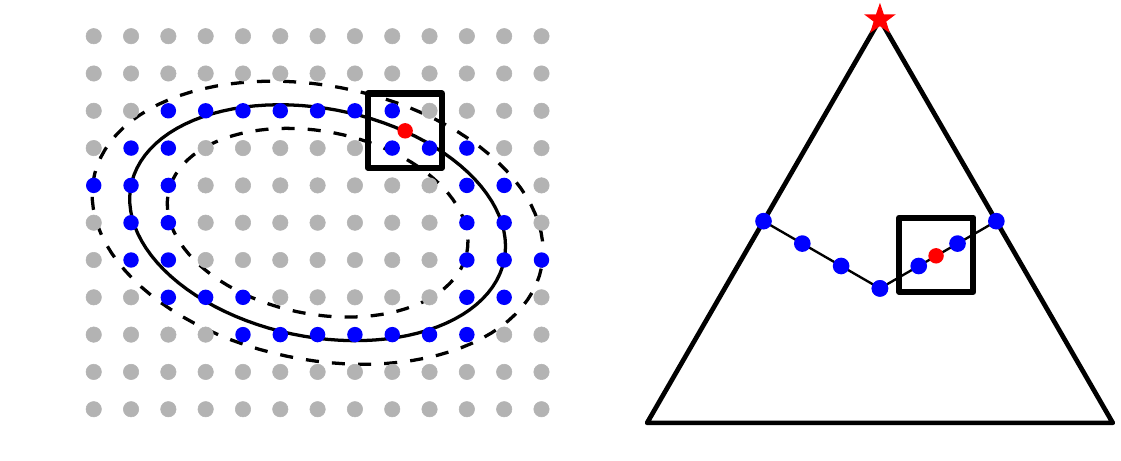}
    \vspace{-16pt}
  \caption{Sampling to cover the level sets of the Lyapunov functions. 
  \label{fig:verification}
  }
  \vspace{-10pt}
\end{wrapfigure}

\paragraph{Certification procedures.}
The certification procedures are as follows: 1) sample points on the boundary of the forward invariant set (blue dots in \Cref{fig:verification}), and check Lyapunov condition holds on all the sampled points; 2) verify the condition holds in a small neighborhood around those points.

\paragraph{Procedure 1: Sampling techniques.}
Rigorous certification is challenging because it requires the set of samples and their neighborhoods to cover the whole boundary of the forward invariant set (not guaranteed by random sampling).
We construct a set for the quadratic Lyapunov function (\Cref{eq: quadratic_lya}) and the classification Lyapunov function (\Cref{eq: margin_lya}) respectively, and show that the proposed set can cover the level set of the corresponding Lyapunov function in Theorem \ref{thm:sample_boundary} below, i.e. for any point on the level set (red dot in \Cref{fig:verification}), there exists a sampled point nearby (blue dot).

To sample on the level set of the quadratic Lyapunov function $\mathcal{D} = \{ \state \in \mathbbm{R}^n | \state^\top P \state = c \}$, we first create a uniform grid  $\mathcal{G}$ (with spacing $r$) in the ambient space that covers the Lyapunov level set. 
We pick $r$ to be at most $\sqrt{\frac{c}{\lambda_1}}$, where $\lambda_1$ is the maximum eigenvalue of $P$. 
Then we do rejection sampling to keep the points that are close to the Lyapunov level set via $\mathcal{G} \cap \mathcal{B}$ for $\mathcal{B}$ defined below:
\begin{equation}
    \mathcal{B} = \{ \state | \underline{c} \leq \state^\top P \state \leq \overline{c} \} \label{eq:def_B} 
\end{equation}
where $\underline{c} = (\sqrt{c} - \frac{\sqrt{n}}{2} r \sqrt{\lambda_1})^2$ and $\overline{c} = (\sqrt{c} + \frac{\sqrt{n}}{2} r \sqrt{\lambda_1})^2$.
We show that $\mathcal{B} \cap \mathcal{G}$ (points inside dashed lines in \Cref{fig:verification}, Left) covers the $c$-level set of the quadratic Lyapunov function (Theorem \ref{thm:sample_boundary} (a)).

To sample on the 1-level set of the $n$-class classification Lyapunov function (the decision boundary), we consider the following set  $\tilde{S}_y$ (\Cref{fig:verification}, Right)  (Theorem \ref{thm:sample_boundary} (b)):
\begin{equation}
\label{eq:sample_simplex}
    \tilde{S}_y = \{ \tilde{\bms} \in \mathbbm{R}^n | \tilde{\bms} = \frac{\bms}{\den}, \bms \in S_y \}
\end{equation}
where $S_y = \{\bms \in \mathbbm{Z}^n | \sum_{i=1}^n \bms_i = \den, \bms_y = \max_{i \neq y} \bms_i, \bms_i \geq 0, \forall i=1,...,n \} $, and $\den$ represents sample density, and needs to be a positive even integer and $\den \not \equiv 1$ (mod $n$).

\begin{theorem}\normalfont{(Sampling on the boundary of a FI set).}
\label{thm:sample_boundary}
\begin{subtheorems}
\item[(a) ] For any $\bmeta \in \mathcal{D}$, there exist an $\bms \in \{ \mathcal{B} \cap \mathcal{G} \}$ such that $|\bmeta_i - {\bms}_i | \leq \frac{r}{2}$ for all $i=1, ..., n$. 
\item[(b) ] For any $\bmeta \in \mathcal{D}_y$, there exists an $\tilde{\bms} \in \tilde{S}_y$ such that $|\bmeta_i - \tilde{\bms}_i | \leq \frac{1}{\den}$ for all $i=1, ..., n$. 
\end{subtheorems}
\end{theorem}

\paragraph{Procedure 2: Verification in a neighborhood around the sampled points.}
Since we only sample a finite number of points on the level set, certifying robust forward invariance requires verifying that the condition holds in a small neighborhood around each of the points. 
We do so by bounding the range of the output given the range of the input.
This bound can be obtained by estimating the Lipschitz constant of the LHS of \Cref{eq:robust_fi}, and the norm of output difference can be bounded by the norm of the input difference. This is convenient for cases where it is simple to bound the Lipschitz of the Lyapunov function and the dynamics. For instance, for classification problems, the Lipschitz constant of the Lyapunov function (\ref{eq: margin_lya}) is $\sqrt{2}$, and the Lipschitz constant of the dynamics with respect to both $\bmeta$ and $\bmx$ is 1 because we use orthogonal layers in the neural network.
For more general Lyapunov functions and dynamics, the Lipschitz bound is often either intractable or vacuous. Instead, we use a popular linear relaxation based verifier CROWN \citep{zhang2018efficient} to bound the output of any general computation graph. In the control case, we verify both $\ins$ and  $\bmeta$ since we want the robustness robustness for a set of perturbations over a set in the state-space.

\section{Experiments}
Our main evaluation is in an application of certified robust forward invariance in nonlinear continuous control (\Cref{sec: results_control}). 
 We also explore the generality of our approach by studying  a second application in certified robustness for image classification (\Cref{sec: results_certified_robustnes}). 

\subsection{Certifying Safety for Robust Continuous Control}
\label{sec: results_control}
\paragraph{Setup.}
We evaluate our framework on a planar segway system, which is a highly unstable nonlinear system whose dynamics is sensitive to its system parameters and therefore hard to train certifiably robust nonlinear controllers  (see \Cref{sec:appen_segway} for the details).
We train a neural network controller (a 3-layer multi-layer perceptron (MLP)) to keep the system forward invariant within the 0.15-sublevel set of a jointly learned Lyapunov function under $\pm 2\%$ perturbations on each system parameter. 
This guarantees that the segway will not fall under adversarial system perturbations.
We evaluate the its performance under both nominal and adversarial system parameters for 1000  adversarially selected initial states within the safe set. The adversarial parameters and states are optimized jointly via projected gradient descent for 100 steps to maximize violations of the forward invariance condition. We also provide provable certificates using certification approach in \Cref{sec: verification}.

\paragraph{Results.}
 We compare with two competitive robust control baselines, and perform ablation studies on different training algorithms in Table \ref{tab:robust_control_results}. 
 We make three main observations.  
 First, it is difficult to certify even non-robust forward invariance using previous methods, highlighting the need for more advanced methods. 
 Second, our robust Lyapunov training approach (Algorithm \ref{alg:training}) is able to train NODE controllers with robust forward invariance certificates.  Third, certifying non-robust forward invariance is easier than certifying robust forward invariance. 
 Overall, these results suggest that our approach is able to train nonlinear ODE controllers in non-trivial settings where existing approaches cannot. To our knowledge, this is also the first instance of training NODE policies with such non-vacuous certified guarantees.
 
\begin{table}[h]
\vspace{-3pt}
  \centering
\caption{Robustness of controllers trained with different methods. The numbers are the percentage of trajectories that stay within the forward invariant set under the nominal and adversarial system parameters on 1000 adversarially selected initial states. The certificate column indicates whether the (robust) FI property is certified.}  
  \resizebox{0.85\linewidth}{!}{  
  \begin{tabular}{lcccc}
    \toprule
    \multirow{2}{*}{\textbf{Method}} & \multicolumn{2}{c}{\textbf{Empirical}} & \multicolumn{2}{c}{\textbf{Certificate}}  \\
    \cline{2-5}
    & \textbf{Nominal} & \textbf{Adv} & \textbf{FI} & \textbf{Robust FI} \\
    \midrule
    Robust LQR & 96.2 & 92.8 & \xmark & \xmark \\
    Robust MBP \citep{donti2020enforcing} & 94.3 & 93.6 & \xmark & \xmark \\
    \midrule
    Standard Backprop Training & 58.0 & 50.4 & \xmark & \xmark \\
     Basic Lyapunov Training \citep{jimenez2022lyanet} & 90.2 & 52.6 & \xmark & \xmark \\     
    \  + Adaptive Sampling & 100 &  68.9 & \cmark & \xmark \\
      \ + Adversarial Training & 100  & 97.8 & \cmark & \xmark \\       
    \textbf{\ + Both (Robust FI-ODE, Ours)} & 100  & 100 & \cmark & \cmark \\ 
    \bottomrule
  \end{tabular}
  }
  \label{tab:robust_control_results}
\end{table}

\paragraph{Visualizations.}
We show trajectories that start within the safe set (gray ellipse) and the corresponding Lyapunov functions in \Cref{fig:segway}. 
The left shows the trajectories of a certifiably \emph{non-robust} FI controller. While the system is safe $100 \%$ under nominal system parameters, it fails for adversarial system parameters. The right shows the trajectories of a certifiably \emph{robust} FI controller. Even under adversarial system parameters, it keeps the trajectories within the safe set $100 \%$ of the time.

\begin{figure}[h]
    \includegraphics[width=\linewidth, trim={0.1cm 0 0 0},clip]{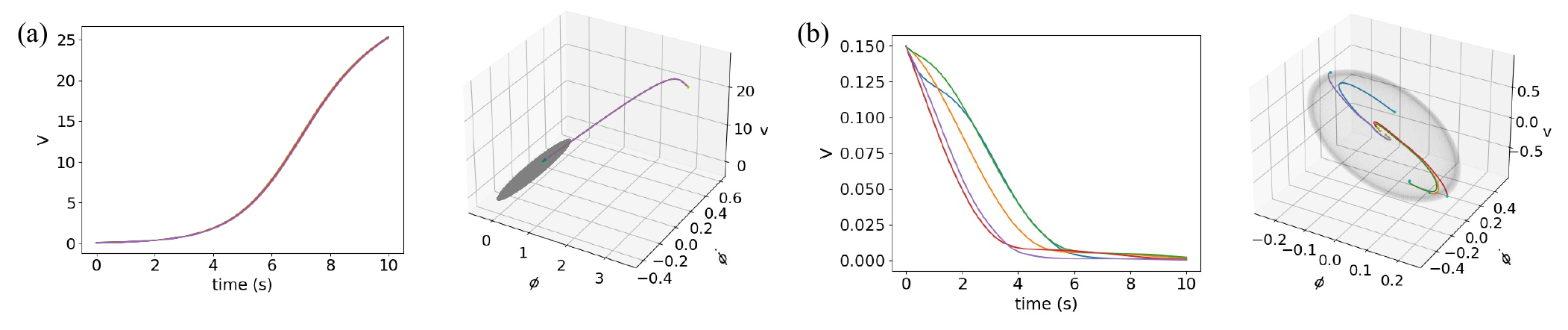}
    \caption{Showing Lyapunov function value $V$ along the trajectories of a planar segway. The forward invariant set is the 0.15-sublevel set. All trajectories start within the forward invariant set (gray ellipse). Each system parameter are perturbed adversarially within $\pm 2 \%$ of their original value.  (a) Shows the Lyapunov function values and system trajectories of a certifiably \emph{non-robust} FI controller. (b) Shows the Lyapunov function values and system trajectories of a certifiably \emph{robust} FI controller. }
    \label{fig:segway}
\end{figure}

\subsection{Certified Robustness for Image Classification}
\label{sec: results_certified_robustnes}
We also apply our approach to train certifiably robust NODEs for image classification to explore the generality of the framework. We treat the input image as system parameters, and set all the initial states to be $\bmeta(0)=\mathbbm{1}\frac{1}{n}$. 
Table \ref{tab:certify-robustness} shows the results. The main metric is certified accuracy, the percentage of test set inputs that are certifiably robust. Our approach achieves the strongest overall certified robustness results compared to prior ODE-based approaches. We also reported clean and adversarial accuracy for references. In addition, the ablation studies shows similar trends as in the robust control experiments: all the learning components: Lyapunov training, Lipschitz restriction and adaptive sampling are needed for good performance, and the model that is trained with all of them (Robust FI-ODE) achieved the highest certified accuracy.

\begin{table}[h]
 \caption{Evaluating certified robustness for image classification. $\epsilon$ is the $\ell_2$ norm of the input perturbations. We report the classification accuracy ($\%$) on clean \& adversarial inputs, and the percentage of inputs that are certifiably robust (Certified). Semi-MonDeq results are on 100 test images [95\% CI in bracket] due to high cost, and other results are on all test images (10,000). }
 \label{tab:certify-robustness}
\centering
\resizebox{0.9\linewidth}{!}{
\begin{tabular}{llcccc}
\toprule
\textbf{Dataset} & \textbf{Method} & $\mathbf{\epsilon}$ & \textbf{Clean} & \textbf{Adversarial} & \textbf{Certified} \\
\hline
 & Lipschitz-MonDeq \citep{pabbaraju2020estimating} & 0.1 & 95.60 & 94.42 & 83.09  \\
 & Semi-MonDeq \citep{chen2021semialgebraic} $^\dagger$ & 0.1 & 99 [>94] & 99 [>94] & 99 [>94] \\
\small{MNIST} &  \textbf{Robust \ours (Ours)} & 0.1 & 99.35 & 99.09 & 95.75 \\
\cmidrule{2-6}
 & Lipschitz-MonDeq \citep{pabbaraju2020estimating} & 0.2 & 95.60 & 93.09 & 50.56 \\
& \textbf{Robust \ours (Ours)} & 0.2 & 99.35 & 98.83 & 81.65 \\
\midrule
 & Lipschitz-MonDeq \citep{pabbaraju2020estimating} & 0.141 & 66.66 & 50.51 & <7.37  \\
 & NODE w/o Lyapunov training & 0.141 & 69.05 & 56.94 & 16.81  \\
\small{CIFAR-$10$} & LyaNet \citep{jimenez2022lyanet} + Lipschitz restriction & 0.141 & 73.15 & 64.87 & 41.43 \\
 & LyaNet \citep{jimenez2022lyanet} + Sampling scheduler & 0.141 & 82.83 & 74.81 & 0 \\
 & \textbf{Robust \ours (Ours)} & 0.141 & 78.34 & 67.45 & 42.27  \\
\bottomrule
\end{tabular}
}
\end{table}
\section{Related Works}\label{sec:related}

\paragraph{Robust control and robust learning-based control methods.} 
Robust control involves creating feedback controllers for dynamic systems that sustain performance under adverse conditions \citep{zhou1998essentials, bacsar2008h}, often relying on basic (linear) controllers.
Our work learns \emph{nonlinear} controllers parameterized by neural networks, while maintaining the robust forward invariance guarantees. There have been recent works for learning-based control with robustness guarantees, such as focusing on  $\mathcal{H}_{\infty}$ robust control \citep{abu2006policy,luo2014off,friedrich2017robust,han2019h,zhang2020policy}, or linear differential inclusions systems \citep{donti2020enforcing}. In comparison, our framework could be used for general nonlinear systems and norm-bounded input/system parameter perturbations.

\paragraph{Learning Lyapunov functions and controllers for nonlinear control problems.}
Various studies focus on learning neural network Lyapunov functions, barrier functions, and contraction metrics for nonlinear control \citep{dawson2023safe}. For stability or safety certification, \citet{chang2019neural} employ SMT solvers \citep{gao2013dreal}, \citet{jin2020neural} use Lipschitz methods, and \citet{dai2021lyapunov} apply mixed integer programming. Our approach uses a linear relaxation-based verifier \citep{zhang2018efficient}, balancing tightness and computational efficiency, to certify nonlinear control policies (unlike the linear policies in \citet{chang2019neural, jin2020neural}) on actual dynamics, contrasting with \citet{dai2021lyapunov}'s neural network dynamic approximations.

\paragraph{Verification and Certified robustness of NODEs.}
Many studies (e.g., \citet{yan2019robustness, kang2021stable, huang2022adversarial}) demonstrate improved empirical robustness of NODEs, yet certifying this robustness is challenging. Prior NODE analyses primarily address reachability:
\citet{grunbacher2021verification} proposes a stochastic bound on the reachable set of NODEs, while
\citet{lopez2022reachability} computes deterministic reachable set of NODEs via zonotope and polynomial-zonotope based methods implemented in CORA \citep{althoff2013reachability}.
However, these methods are limited to low-dimension or linear NODEs.
MonDEQ \citep{winston2020monotone}, akin to implicit ODEs, has seen $\ell_2$ robustness certification efforts \citep{pabbaraju2020estimating, chen2021semialgebraic}, but these struggle beyond MNIST.
\citet{xiao2022forward} propose invariance propagation for stacked NODEs that provides guarantees for output specifications by controller/input synthesis. While their approach focuses more on interpretable causal reasoning of stacked NODEs, our work provides a framework for training and provably certifying general NODEs.

\paragraph{Formal verification of neural networks.}
Formal verification of neural networks aim to prove or disprove certain specifications of neural networks, and a canonical problem of neural network verification is to bound the output of neural networks given specified input perturbations. Computing the exact bounds is a NP-complete problem~\citep{katz2017reluplex} and can be solved via MIP or SMT solvers~\citep{tjeng2017evaluating,ehlers2017formal}, but they are not scalable and often too expensive for practical usage. In the meanwhile, incomplete neural network verifiers are developed to give sound outer bounds of neural networks~\citep{salman2019convex,dvijotham2018dual,wang2018efficient,singh2019abstract}, and bound-propagation-based methods such as CROWN~\citep{zhang2018efficient} are a popular approach for incomplete verification. Recently, branch-and-bound based approaches~\citep{bunel2020branch,wang2021beta,DePalma2021} are proposed to further enhance the strength of neural network verifiers. Our work utilizes neural network verifiers as a sub-procedure to prove forward invariance of NODEs, and is agnostic to the verification algorithm used. We used CROWN because it is efficient, GPU-accelerated and has high quality implementation~\citep{xu2020automatic}.
\section{Conclusion and Future Work}
We introduced FI-ODE, a framework ensuring certifiable robust forward invariance in NODEs. Our work showcases certified robustness in both nonlinear NODE control and image classification, marking a step towards certifying complex NODEs across various domains. Future directions include extending our method to higher-dimensional NODEs and more general Lyapunov functions.

\acks{This work is funded in part by AeroVironment and NSF \#1918865.}

\bibliography{reference}

\newpage
\appendix
\onecolumn


\section{Definitions for Class \texorpdfstring{$\mathcal{K}$}{K} functions}

\begin{definition}[Class $\mathcal{K}$ Function]
A continuous function $\alpha: [0, a) \to [0, \infty)$ for $a \in \BR_{>0} \cup \{\infty\}$ belongs to class $\mathcal{K}$ ($a \in \mathcal{K}$) if it satisfies:
\begin{enumerate}
    \item \textbf{Zero at Zero: } $\alpha(0) = 0$
    \item \textbf{Strictly Increasing: }For all $r_1,r_2 \in [0, a]$ we have that $r_1 < r_2 \Rightarrow \alpha(r_1) < \alpha(r_2)$
\end{enumerate}
\label{def:k}
\end{definition}
\begin{definition}[Class $\mathcal{K}_\infty$ Function]
A function belongs to $\mathcal{K}_\infty$ if it satisfies:
\begin{enumerate}
    \item $\alpha \in \mathcal{K}$
    \item \textbf{Radially Unbounded: }$\lim_{r \to \infty} \alpha(r) = \infty$
\end{enumerate}
\label{def:k_inf}
\end{definition}

\begin{definition}[Extended Class $\mathcal{K}_\infty^e$ Function]
A continuous function $\alpha: \BR \to \BR$ belongs to extended $\mathcal{K}^e_\infty$ if it satisfies:
\begin{enumerate}
    \item \textbf{Zero at Zero: } $\alpha(0) = 0$
    \item \textbf{Strictly Increasing: }For all $r_1,r_2 \in [0, a]$ we have that $r_1 < r_2 \Rightarrow \alpha(r_1) < \alpha(r_2)$
\end{enumerate}
\label{def:k_inf_e}
\end{definition}

\begin{definition}[Class $\mathcal{KL}$ Function]
A continuous function $\beta: [0,a) \times [0, \infty) \to [0, \infty)$ belongs to $\mathcal{KL}$ if it satisfies:
\begin{enumerate}
    \item \textbf{Class $\mathcal{K}$ on first argument: } $\forall s \in [0, \infty) \beta(\cdot,s) \in \mathcal{K}$
    \item \textbf{Asymptotically $0$ on second argument: } $\forall r \in [0,a) \lim_{s \to \infty} \beta(r,s) = 0$
\end{enumerate}
\label{def:kl}
\end{definition}

\section{Forward Invariance on a Probability Simplex}
\label{sec: simplex-arch}

\begin{wrapfigure}{R}{0.3\linewidth} 
    \includegraphics[width=1\linewidth]{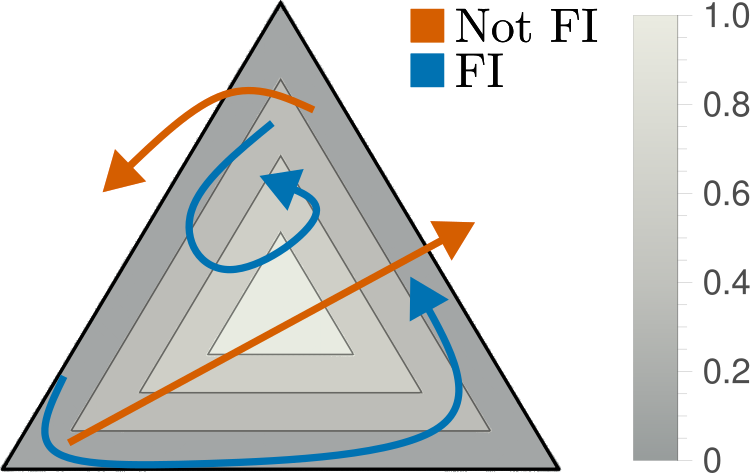}
  \caption{The color contours show level-sets of a barrier function in a 3-class probability simplex.
  \label{fig:CBF}
  }
\end{wrapfigure}

For the purposes of certification and training, it is often useful to make the state space $\mathcal{H}$ be a bounded set, as certifying over unbounded sets is typically intractable. 
For multi-class classification, a natural choice  is the probability simplex. 
Since we initialize $\bmeta$ within the simplex, it suffices to render the simplex to be forward invariant. 
We explicitly constrain the states to a probability simplex using a Control-Barrier Function based Quadratic Program (CBF-QP)\footnote{This is analogous to projected gradient descent (PGD) where we project the dynamics instead of the states.} \cite{ames2016control}, implemented as a differentiable optimization layer \cite{cvxpylayers2019}.

Barrier functions can be viewed as a variant of Lyapunov functions that only require the state to stay within a set rather than  always make progress towards some minimum. 
Specifically, we choose a potential function $h$ with a 0-super level set (i.e. $\left\{ \state \in \mathcal{H}  |h(\state) \geq 0 \right\}$) equal to the desired forward invariant set $\mathcal{S}$ (see \Cref{fig:CBF} for an example).
Similarly to the Lyapunov case, there is a point-wise inequality condition that must be true over the forward invariant set: 
\begin{align}
    \frac{d}{dt} h(\bm{\eta}(t)) \geq -\alpha(h(\bm{\eta}))
    \label{eq:short_barrier}
\end{align}
where $\alpha: \BR_{\geq 0} \to \BR_{\geq 0}$ is a class $\mathcal{K}_\infty$ function.
Intuitively, all the flows on the boundary of the forward invariant set must have a positive time-derivative (otherwise there could be a point on the boundary that decreases the value of $h$ and thus exits the forward invariant set).
This is the essence of Nagumo's theorem \cite{nagumo1942lage}. 
Barriers extend this idea with a condition that can be applied everywhere in the target forward invariant set without being overly conservative.
As trajectories approach the boundary of the set, \Cref{eq:short_barrier} ensures the time derivative increases until it is positive at the boundary. We use a variation of barrier functions called Control Barrier Functions (CBF). We formalize this concept with \Cref{thm:cbf}.

In our case, the unconstrained dynamics is the output of a neural network and we denote it as $\hat{f}(\bmeta, \bmx)$. To make the dynamics satisfy the barrier conditions, we use a Control Barrier Function Quadratic Program (CBF-QP) Safety Filter \cite{gurriet2018towards}: 
\begin{subequations}
\begin{align}
    f(\hat{\bmf}) &= \argmin_{\mathtt{f} \in \BR^n} \frac{1}{2}\lVert \mathtt{f} - \hat{\bmf} \rVert_2^2 \label{eq:cbfqp_a}\\
    & \text{s.t} \quad \quad \mathbbm{1}^\top \mathtt{f} = 0  \label{eq:cbfqp_b}\\
    & \quad \quad \quad \mathtt{f} \geq - \alpha(\bmeta)  \label{eq:cbfqp_c}
\end{align}
\end{subequations}
where the arguments to the function $\hat{f}$ are omitted for brevity. 

Recall that an $n$-class probability simplex is defined as $\triangle = \{ \bmeta \in \mathbbm{R}^n | \sum_{i=1}^n \bmeta_i=1, \bmeta_i \geq 0 \}$. 
Now we show that \Cref{eq:cbfqp_b} ensures that the sum of the state stays to be 1 and \Cref{eq:cbfqp_c} guarantees the state to be non-negative. 

First, we need the sum of $\bmeta$ stays the same as the initial condition. Taking time derivative of both sides of $\sum_{i=1}^n \bmeta_i=1$, we have $\frac{d}{dt}\left( \sum_{i=1}^n \bmeta_i(t)\right) = \sum_{i=1}^n f_\bmtheta(\bmeta, \bmx)_i = \mathbbm{1}^\top f_\bmtheta(\bmeta, \bmx) = 0$, which is \Cref{eq:cbfqp_b}. This is natural because the dynamics summing up to zero means the changes from all dimensions summing up to zero, and thus the sum of all dimensions stays the same.

Next, we need each dimension of the state to be non-negative. Since the initial condition has non-negative entries, we just need the set $\{ \bmeta | \bmeta \geq 0 \}$ to be forward invariant. We define forward invariance via barrier functions. For each dimension $i$, we define $h_i(\bmeta) = \bmeta_i$. Then the 0-superlevel set of $h_i$ equals the safe set $\{ \bmeta | \bmeta_i \geq 0 \}$.
As long as the condition in \Cref{eq:short_barrier} holds, i.e. $\frac{d h_i}{d\bmeta}f_\params(\bmeta, \bmx) \geq - \alpha (h_i(\bmeta))$ for some class $\mathcal{K}_\infty$ function $\alpha$, the set $\{ \bmeta | \bmeta_i \geq 0 \}$ is forward invariant.
Plugging in $h_i(\bmeta)$, we have $f_\params(\bmeta,\bmx) \geq -\alpha( \bmeta)$, which is \Cref{eq:cbfqp_c}.

To learn in this setting we differentiate through the QP layer using the KKT conditions as shown in \cite{cvxpylayers2019}. Given the simple nature of the QP, we implemented a custom solver that uses binary search to efficiently compute solutions, detailed in the supplymentary materials.

To demonstrate the effectiveness of the CBF-QP layer, we visualize the learned trajectories on the CIFAR-3 dataset (a subset of CIFAR-10 with the first 3 classes) in \Cref{fig:trajectories}. Each colored line represents a trajectory of an input image from a specific class.
As training progresses, the trajectories are trained to evolve to the correct classes. All the trajectories stay in the simplex, implying that the learned dynamics satisfy the constraints in \Cref{eq:cbfqp_b} and \Cref{eq:cbfqp_c}.

\begin{figure}[h]
\centering
\includegraphics[width=0.6\linewidth]{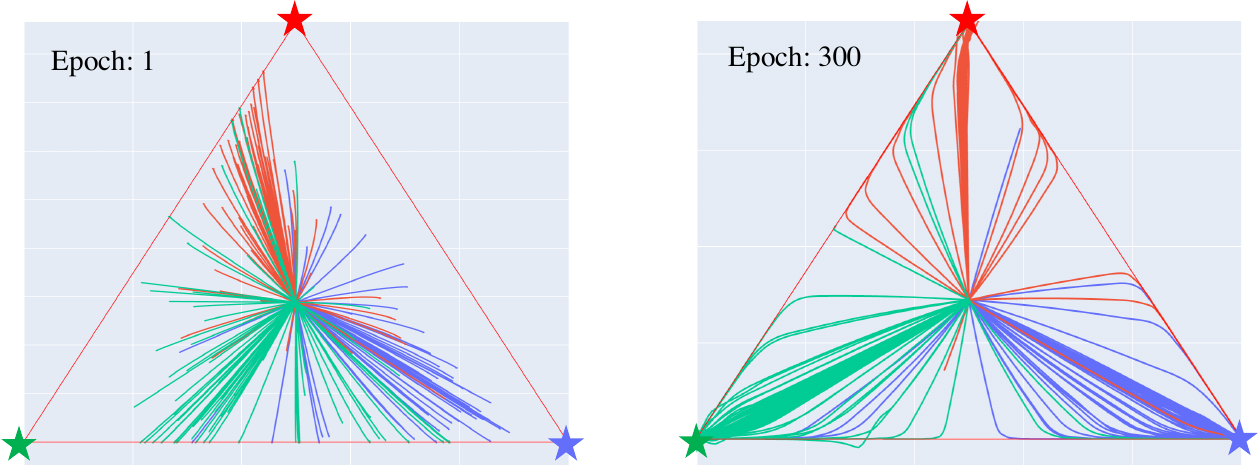}
\caption{Depicting ODE trajectories that satisfy the simplex constraint for CIFAR-3 on epochs 1 and 300. Each colored line represents the trajectory of an input example of a specific class, and the stars at the corners are colored with the ground-truth class.}
\label{fig:trajectories}
\end{figure}

\section{Theorems with Proof}
\begin{theorem}[ES-CLF Implies Exponential Stability \cite{ames2014rapidly}]
\label{thm:es-clf}
For the ODE in \cref{eq:neuralode_input,eq:neuralode_hidden}, a continuously differentiable function $V: \BR^k \to \BR_{\geq 0}$ is an Exponentially Stabilizing Control Lyapunov Function (ES-CLF) if there are class $\mathcal{K}_\infty$ functions $\overline{\sigma}$ and $\kappa$ such that:
\begin{gather}
    V(\bm{\eta}) \leq \overline{\sigma}(\lVert \bm{\eta} \rVert),\\
    \min_{\bm{\theta} \in \Theta} \left[ \frac{\partial V}{\partial \bm{\eta} }^\top \bm{f_\params}(\bm{\eta}, \bm{x}) + \kappa(V(\bm{\eta})) \right] \leq 0 \label{eq:lya_cond}
\end{gather}
holds for all $\bm{\eta}\in \mathcal{R} \subseteq H$ and $t \in [0, 1]$. The existence of an ES-CLF implies that there is a $\bm{\theta} \in \Theta$ that can achieve:
\begin{align}
    \frac{\partial V}{\partial \bm{\eta}}^\top \bm{f_\params}(\bm{\eta}, \bm{x}) + \kappa (V(\bm{\eta})) \leq 0,
    \label{eq:lyapunov_achieved_bound}
\end{align}
and furthermore the ODE using $\params$ is exponentially stable with respect to $V$, i.e., $V(\state(t)) \leq V(\state(0))e^{-\underline{\kappa} t}$ for some $\underline{\kappa} > 0$.
\label{thm:lyapunov}
\end{theorem}

\begin{proof}
    Since $\Theta$ is compact, minimums are attained within $\Theta$. Let $\sparams(\bm{\eta}) = \argmin_{\bm{\theta} \in \Theta} \frac{\partial V}{\partial \bm{\eta} }^\top \bm{f_\params}(\bm{\eta}, \bm{x})$. Therefore, from \cref{eq:lya_cond} we can conclude that:

    \begin{align}
        \frac{\partial V}{\partial \bm{\eta} }^\top \bm{f}_{\sparams(\bm{\eta})}(\bm{\eta}, \bm{x}) \leq - \kappa (V(\bm{\eta})) & \quad \quad \forall \bm{\eta}\in \mathcal{H}
    \end{align}
    For simplicity we will omit the arguments of $\sparams$. Furthermore, in the case where $\sparams$ is a set we will select only one.
    Since $\mathcal{H}$ is compact then
    \begin{align}
        \sigma^* = \max_{\state \in \mathcal{H}} \overline{\sigma}(\lVert \state \rVert).
    \end{align}
    This in turn implies that $V$ in bounded in $\mathcal{H}$ which helps us conclude that 
    \begin{align}
        \overline{V} = \max_{\bm{\eta} \in \mathcal{H}}V(\bm{\eta})
    \end{align}
    is well defined which in turn implies 
    \begin{align}
        \underline{\kappa} = \min_{r \in [0, \overline{V}]} \frac{d \kappa(r)}{dr}
    \end{align}
     is well defined.
     Since $\kappa$ is strictly increasing, then $\underline{\kappa} > 0$. Notice that $\alpha(r) = \underline{\kappa} r$ satisfies $\alpha \in \mathcal{K}_\infty$ and, by the comparison lemma, $\forall r, \underline{\kappa} r \leq \kappa(r) $. Therefore:
     \begin{equation}
        \frac{\partial V}{\partial \bm{\eta} }^\top \bm{f}_{\sparams(\bm{\eta})}(\bm{\eta}, \bm{x}) \leq - \kappa (V(\bm{\eta})) \leq - \underline{\kappa} V(\bm{\eta}), \quad \quad \forall \bm{\eta}\in \mathcal{H} 
    \end{equation}
    In preparation for applying the comparison lemma we will consider the following Initial Value Problem (IVP):

    \begin{align}
        y(0) &= V(\bm{\eta}_0) \label{eq:cl_in_init}\\
        \dot{y} &= -\underline{\kappa}y \label{eq:cl_in_dyn}
    \end{align}
    Since this is a linear system solutions for $y$ exist, are unique and take the form $y(t) = V(\bm{\eta}_0) e^{-\underline{\kappa} t}$.
    Furthermore, by the comparison lemma we can conclude that:
    \begin{align}
        V(\state(t)) \leq y(t) = V(\bm{\eta}_0) e^{-\underline{\kappa} t}
    \end{align}

\end{proof}

\begin{lemma}[Solution of Class $\mathcal{K}$ function systems (See Lemma 4.4 in \cite{khalil2002nonlinear})]
Let $\alpha \in \mathcal{K}_\infty^e$. Then consider the following IVP for $t \in [0,1]$:
\begin{align}
    y(0) &= y_0 \\
    \dot{y} &= -\alpha(y)
\end{align}
This IVP has unique solutions $y(t) = \beta(y_0,t)$ where $\beta \in \mathcal{KL}$.
\label{thm:cl_soln_beta}
\end{lemma}

\begin{theorem}[CBF Existence Implies Forward Invariance \cite{xu2015robustness,nagumo1942lage}]
    \label{thm:cbf}
    Let the set $\mathcal{S} \subset \mathcal{H}$ be the 0 superlevel set of a continuously differentiable function $h: \mathcal{H} \to \BR$, i.e. $\mathcal{S} = \{ \state \in \mathcal{H} \rvert h(\state) \geq 0  \}$.  The set $\mathcal{S}$ is forward invariant with respect to the ODE \cref{eq:neuralode_input,eq:neuralode_hidden}, if $h$ is a Control Barrier Function (CBF) i.e. it satisfies either of the following  conditions:
    \begin{enumerate}
        \item \cite{xu2015robustness} There exists a function $\alpha$ for all $\state \in \mathcal{S}$ so that:
        \begin{gather}
        \label{eq:cbf}
        \max_{\params \in \Theta} \left[ \frac{\partial h}{\partial \state}^\top \dyn_\params(\state, \ins) + \alpha( h(\state)) \right] \geq 0,
        \end{gather}
        where $\alpha$ is a class $\mathcal{K}_\infty^e$ function (this means $\alpha : \BR \to \BR$ is strictly increasing and satisfies $\lim_{r \to \infty} \alpha(r) = \infty$).
        \item \cite{nagumo1942lage} For all $\state \in \{ \state \in \mathcal{S} \rvert h(\state) = 0 \}$:
        \begin{subequations}\label{eq:nagumo}
        \begin{align}
            & \frac{\partial h}{\partial \state } \neq \bm{0}  \\
            & \max_{\params \in \Theta} \left[ \frac{\partial h}{\partial \state}^\top \dyn_\params(\state, \ins) \right]\geq 0.  
        \end{align}
        \end{subequations}
    \end{enumerate}
    
    \end{theorem}

\begin{proof}
    Both conditions follow from fundamentally different arguments. Condition 1 follows from the comparison lemma. Condition 2 uses Nagumo's theorem. In either case we rely on the compactness of $\Theta$ to solve the following optimization problem:
    \begin{align}
        \sparams(\state) = \argmax_{\params \in \Theta} \left[ \frac{\partial h}{\partial \state}^\top \dyn_\params(\state, \ins) \right]
    \end{align}
    We will omit the parameters of $\sparams$ for brevity and choose a random solution in the case where $\sparams$ returns a set of solutions.

    \begin{enumerate}
        \item Consider the following IVP:
        \begin{align}
            y(0) &= h(\state(0)) \\
            \dot{y} &= -\alpha(y)
        \end{align}
        which satisfies the conditions of \cref{thm:cl_soln_beta}. This implies that $y(t)$ is unique and $y(t) = \beta(h(\state(0)), t)$ where by the assumption of forward invariance $h(\state(0)) \geq 0$. The by a trivial variant of the comparison lemma we have that $\dot{h}(\state(t)) \geq - \alpha(h(\state))$ implies $h(\state(t)) \geq \beta(h(\state(0)), t)$ which implies $h(\state(t)) \geq 0$ for $t \in [0,1]$
        \item This is a direct application of Nagumo's theorem \cite{nagumo1942lage}.
    \end{enumerate}
    
\end{proof}

\subsection{Proof of \texorpdfstring{\Cref{thm:robust_FI}}{Robust FI Section}}
\label{proof:robust_classification}
\begin{proof}

Define barrier function $h$ as $h = c - V$, then $\mathcal{S}$ is a 0-superlevel set of $h$. According to Nagumo's theorem \cite{nagumo1942lage}, if $\dot{h} (\ins+\bmeps; \bmeta) \geq 0$ on $\partial \mathcal{D} \coloneqq \mathcal{S}$, then $\mathcal{S}$ is forward invariant. Since $\dot{V} (\bmeta; \ins) \leq - L_V L_f^x \eps$ on $\mathcal{D}$, we have $\dot{h}(\bmeta; \ins) \geq L_h L_f^x \eps$ on $\mathcal{D}$ (where $L_h$ is the Lipschitz constant of $h$ and notice that $L_V=L_h$). Then for the perturbed input, we have 
\begin{align}
    \dot{h}(\bmeta; \ins+\bmeps) &= \dot{h}(\bmeta;\ins) + \dot{h}(\bmeta;\ins+\bmeps) - \dot{h}(\bmeta;\ins) \\
    & \geq \dot{h}(\bmeta;\ins) - \| \dot{h}(\bmeta;\ins+\bmeps) - \dot{h}(\bmeta;\ins) \| \\
    &\geq \dot{h}(\bmeta; \ins) - L_h L_f^x \eps \\
    & \geq \kappa \underline{V} - L_h L_f^x \eps \geq 0 
\end{align}
Therefore, $\mathcal{S}$ is still forward invariant for the perturbed inputs with perturbation magnitude smaller than $\eps$.
\end{proof}

\subsection{Proof of \texorpdfstring{\Cref{thm:sample_boundary}}{Sample Simplex Theorem (b)}}
\label{proof:sample_boundary}
\begin{proof}
\textbf{(a) Sampling on the level set of a quadratic Lyapunov function.} Consider the Lyapunov function of the form $V(\bmeta) = \bmeta ^\top P \bmeta$, where $P$ is a positive definite matrix. Let $\mathcal{D} = \{ \state \in \mathbbm{R}^n | \state^\top P \state = c \}$ be the $c$-level set of the Lyapunov function, and let $\mathcal{G}$ be a uniform grid (with spacing $r$) that covers the Lyapunov level set, i.e. $\min_{\bmeta \in \mathcal{D}} \bmeta_i \geq  \min_{\bmeta \in \mathcal{G}} \bmeta_i$ and $\max_{\bmeta \in \mathcal{D}} \bmeta_i \leq  \max_{\bmeta \in \mathcal{G}} \bmeta_i$, $\forall i = 1, ..., n$. Then $\forall \bmeta \in \mathcal{D}$, there exists at least a point $\bm{g} \in \mathcal{G}$ such that $|\bmeta_i - {\bm{g}}_i | \leq \frac{r}{2}$ for all $i=1, ..., n$.
Let $\mathcal{G}_D$ denote those grid points that are close to the decision boundary, i.e. $\mathcal{G}_D = \{ \bm{g} \in \mathcal{G} | |\bmeta_i - {\bm{g}}_i | \leq \frac{r}{2}, \text{ for all } i=1, ..., n, \forall \bmeta \in \mathcal{D} \}$. 

Now we show that $\mathcal{G}_D \subseteq \{ \mathcal{B} \cap \mathcal{G} \}$, where $\mathcal{B}$ is defined in \Cref{eq:def_B}. 
Notice that the maximum $\ell_2$ distance between a point $\bmeta \in \mathcal{D}$ and its closest point $\bm{g} \in \mathcal{G}_D$ is $\frac{\sqrt{n}}{2}r$. Then we find the maximum and minimum Lyapunov function value by perturbing $\bmeta \in \mathcal{D}$ within $\frac{\sqrt{n}}{2}r$ distance. 

Consider the following maximization problem:
\begin{subequations}
\begin{align}
    & \max \, (\bmeta + \bm{v})^\top P (\bmeta + \bm{v}) \\
    & \text{s.t.} \quad  \bmeta^\top P \bmeta = c \\
    & \quad \quad \quad \bm{v}^\top \bm{v} = d^2 
\end{align}
\end{subequations}
Let the eigenvalue decomposition of $P$ be $P=U \Lambda U^\top$, where $U$ is an orthonormal matrix, and $\Lambda=\text{diag}\{ \lambda_1, \lambda_2, ..., \lambda_n \}$ with $\lambda_1 \geq \lambda_2 \geq ... \geq \lambda_n$. Then the solution of the above problem equals to the solution of the following problem (rotating the coordinates by $U$):
\begin{subequations}
\begin{align}
    & \max \, (\bmeta + \bm{v})^\top \Lambda (\bmeta + \bm{v}) \\
    & \text{s.t.} \quad  \bmeta^\top \Lambda \bmeta = c \\
    & \quad \quad \quad \bm{v}^\top \bm{v} = d^2 
\end{align}
\end{subequations}
Then we can find the upper bound of the objective by the following:
\begin{align}
    \sum_{i=1}^n (\bmeta_i + \bm{v}_i) \lambda_i &= c + \sum_{i=1}^n \lambda_i \bm{v}_i^2 + 2 \sum_{i=1}^n \lambda_i \bmeta_i \bm{v}_i \\
    & \leq c + \lambda_1 d^2 + 2 \sum_{i=1}^n \lambda_i \bmeta_i \bm{v}_i \\
    & \leq c + \lambda_1 d^2 + 2 \sqrt{\sum_{i=1}^n (\sqrt{\lambda_i} \bmeta_i)^2 \sum_{i=1}^n (\sqrt{\lambda_i} \bm{v}_i)^2} \\
    & = c + \lambda_1 d^2 + 2 \sqrt{c \sum_{i=1}^n \lambda_i \bm{v}_i^2} \\
    & \leq c + \lambda_1 d^2 + 2d \sqrt{c \lambda_1} \\
    & = (\sqrt{c} + d\sqrt{\lambda_1})^2 \coloneqq \overline{c}
\end{align}

Similarly, we can find the lower bound of the objective by:
\begin{align}
    \sum_{i=1}^n (\bmeta_i + \bm{v}_i) \lambda_i &= c + \sum_{i=1}^n \lambda_i \bm{v}_i^2 + 2 \sum_{i=1}^n \lambda_i \bmeta_i \bm{v}_i \\
    & \geq c + \sum_{i=1}^n \lambda_i \bm{v}_i^2 - 2 \sqrt{c \sum_{i=1}^n \lambda_i \bm{v}_i^2} \\
    & = \left(\sqrt{c} - \sqrt{\sum_{i=1}^n \lambda_i \bm{v}_i^2} \right)^2 \\
    & \geq (\sqrt{c} - d\sqrt{\lambda_1})^2 \coloneqq \underline{c}
\end{align}
Therefore, the maximum and minimum Lyapunov function value the points in $\mathcal{G}_D$ can attain are $\overline{c}$ and $\underline{c}$ with $d=\frac{\sqrt{n}}{2}r$, and thus we have $\mathcal{G}_D \subseteq \{ \mathcal{B} \cap \mathcal{G} \}$. 

By definition of $\mathcal{G}_D$, we have that for any $\bmeta \in \mathcal{D}$, there exist an $\bms \in \mathcal{G}_D$ such that $|\bmeta_i - {\bms}_i | \leq \frac{r}{2}$ for all $i=1, ..., n$. Since $\mathcal{G}_D \in \{ \mathcal{B \cap \mathcal{G}} \}$, we have that for any $\bmeta \in \mathcal{D}$, there exist an $\bms \in \{ \mathcal{B} \cap \mathcal{G}\}$ such that $|\bmeta_i - {\bms}_i | \leq \frac{r}{2}$ for all $i=1, ..., n$.

\textbf{(b) Sampling on the decision boundary.}
For any $\bmeta \in \mathscr{D}_y$, let $\bmz=[ \den \bmeta_1, ..., \den\bmeta_n ]$. By definition of $\mathscr{D}_y$, in addition to $\sum_i \bmz_i = \den$, we have 
\begin{align} 
    & \sum_i \bmz_i = \den \label{eq: def_Dy_sum} \\
    & \bmz_y = \max_{j \neq y} \bmz_j  \label{eq: def_Dy_max}
\end{align}
Define $\tilde{\bmz} = [\bmz_1 - \lfloor \bmz_1 \rfloor, ... , \bmz_n - \lfloor \bmz_n \rfloor ]$ to be the vector that contains the fractional part of each element in $\bmz$. 
Then we sort $\tilde{\bmz}$ in a non-decreasing order. For the tied elements that equals to $\tilde{\bmz}_y$, we put $\tilde{\bmz}_y$ as the last. We denote the sorted vector as $\tilde{\bmz}' = [ \tilde{\bmz}_{i_1}, ..., \tilde{\bmz}_{i_n} ]$, where $\tilde{\bmz}_{i_1} \leq ... \leq \tilde{\bmz}_{i_n}$. 
Let $v: \mathbb{Z}^+ \rightarrow \mathbb{Z}^+$ to be a function that maps the indices in $\tilde{\bmz}$ to the indices in $\tilde{\bmz}'$. For instance, if $\tilde{\bmz}_1$ becomes the third element in $\tilde{\bmz}'$, then $v(1) = 3$. If $\tilde{\bmz}_j = \tilde{\bmz}_y$, we have $v(y) > v(j)$. 

Notice that $\sum_i \tilde{\bmz}'_i = \sum_i \tilde{\bmz}_i = \sum_i \bmz_i - \sum_i \lfloor \bmz_i \rfloor = \den - \sum_i \lfloor \bmz_i \rfloor$, and $\sum_i \tilde{\bmz}'_i < n$ since $0 \leq \tilde{\bmz}'_i < 1$ for $i=1,...,n$.
Let $k = n - (\den - \sum_i \lfloor \bmz_i \rfloor)$, we have
\begin{equation}
\label{eq: equal_area}
    \tilde{\bmz}'_1 + ... + \tilde{\bmz}'_k = (1 - \tilde{\bmz}'_{k+1}) + ... + (1 - \tilde{\bmz}'_n)
\end{equation}
Define vector $\bmq$ as follows:
\begin{equation} 
\label{eq: grid_boundary}
\bmq_{i_j} =\left\{
\begin{array}{lll}
    \lfloor \den \bmeta_{i_j} \rfloor , & j = 1, ... , k \\
    \lceil \den \bmeta_{i_j} \rceil , & j = k+1, ..., n
\end{array}
\right.
\end{equation}
Then we have $| \bmq_i - \bmz_i | < 1$ for all $i=1, ..., n$. 
Now we check $\bmq$ satisfies \Cref{eq: def_Dy_sum} and a relaxed version of \Cref{eq: def_Dy_max}.
First, we have $\sum_i \bmq_i = \den$ because of \Cref{eq: equal_area}.
Next, we show $\bmq_y \geq \max_{i \neq y} \bmq_i$ by contradiction. Suppose there exists an index $j$ such that $\bmq_j > \bmq_y$, then it has to be the case where $\lfloor \bmz_j \rfloor = \lfloor \bmz_y \rfloor$ and we take ceiling on $\bmz_j$ and take floor on $\bmz_y$, i.e. $v(j) > k$ and $v(y) \leq k$. This means $\tilde{\bmz}_j > \tilde{\bmz}_y$, because $v$ is the sorted indices of $\tilde{\bmz}$ in a non-decreasing order and this gives $\tilde{\bmz}_j \geq \tilde{\bmz}_y$, and if $\tilde{\bmz}_j = \tilde{\bmz}_y$, we have $v(y) > v(j)$, which is contradictory to $v(j) > k$ and $v(y) \leq k$. Then we have $\bmz_y = \lfloor \bmz_y \rfloor + \tilde{\bmz}_y < \bmz_j = \lfloor \bmz_j \rfloor + \tilde{\bmz}_j$, which is contradictory to \Cref{eq: def_Dy_max}. Therefore, there does not exist a $j$ such that $\bmq_j > \bmq_y$, i.e. $\bmq_y \geq \max_{i \neq y} \bmq_i$. For the cases where $\bmq_y = \max_{i \neq y} \bmq_i$, we have $\bmq \in S_y$, i.e. $\bmq$ is a sampled point.

For the cases where $\bmq_y > \max_{i \neq y} \bmq_i$, we show that we can modify $\bmq$ to $\tilde{\bmq}$ such that $\tilde{\bmq} \in S_y$ and $|\tilde{\bmq}_i - \bmz_i| \leq 1$ for all $i=1, ..., n$. Let $\mathcal{I} = \{ i \in \mathbb{Z}^+ | \bmz \neq y, \bmz_i = \bmz_y \}$ be the set that contains the indices of all runner-up elements in $\bmz$. If $\bmq_y > \max_{i \neq y} \bmq_i$, then we must have $\bmq_y = \lceil \den \bmeta_y \rceil$, and $\bmq_i = \lfloor \den \bmeta_i \rfloor$ for all $i \in \mathcal{I}$. We first let $\tilde{\bmq} = \bmq$, and then pick an $i^*$ from $\mathcal{I}$. Let $\mathcal{J} = \{ j \in \mathbb{Z}^+ | \bmq_j \geq 1, j \neq i^*, j \neq y \}$. Notice that $\bmq_{i^*} + \bmq_y = 2 \lfloor \bmz_y \rfloor + 1$, which is an odd number. Since $\sum_i \bmq_i=\den$ and $\den$ is an even number, $\mathcal{J} \neq \emptyset$. We discuss how to obtain $\tilde{\bmq}$ case by case. 

Case 1: If there exists a $j \in \mathcal{J}$ such that $v(j) > k$, we set $\tilde{\bmq}_j = \lfloor \den \bmeta_j \rfloor$, and set $\tilde{\bmq}_{i^*} = \lceil \den \bmeta_{i^*} \rceil$. 
Then we have $\sum_i \tilde{\bmq} = \sum_{i \neq i^*, i \neq j} \bmq_i + \bmq_{i^*} + 1 + \bmq_j - 1 =\den$ and $|\tilde{\bmq}_i - \bmz_i| \leq 1$ for all $i=1, ..., n$.

Case 2: If $v(j) \leq k$ for all $j \in \mathcal{J}$, there must exist a $j$ such that $\bmq_j < \bmq_{i^*}$. Otherwise, $\bmq_y = \bmq_{i^*} + 1$, and $\bmq_i = \bmq_{i^*}$ for all $i \neq y$. Then $\sum_i \bmq_i = \den = n \bmq_{i^*} + 1$, which is contradictory to the assumption that $\den \not\equiv 1 (\text{mod } n)$. Then set $\tilde{\bmq}_j = \lceil \den \bmeta_j \rceil$ and $\tilde{\bmq}_y = \lfloor \den \bmeta_y \rfloor$. Since $\bmq_j < \bmq_{i^*}$, we have $\tilde{\bmq}_j = \bmq_j + 1 \leq \tilde{\bmq}_y = \bmq_{i^*}$.
\end{proof}

\begin{remark}
The assumption that $\den \not\equiv 1 (\text{mod } n)$ is easy to satisfy. Since we also require $\den$ is an even number, as long as $n$ is also an even number, we have $\den \not\equiv 1 (\text{mod } n)$. We can also relax this assumption by adding $[\frac{\den}{n}, ..., \frac{\den}{n}]$ to $S_y$.
\end{remark}

\subsection{Custom solver for the CBF-QP}
Consider a CBF-QP in the following form:
\begin{align}
\label{eq:supp_qp}
    f(\hat{\bmf}) &= \argmin_{\mathtt{f} \in \BR^n} \frac{1}{2}\lVert \mathtt{f} - \hat{\bmf} \rVert_2^2 \\
    & \text{s.t} \quad \quad \mathbbm{1}^\top \mathtt{f} = b \nonumber \\
    & \quad \quad \underline{\mathtt{f}}(\bmeta) \leq \mathtt{f} \leq \overline{\mathtt{f}}(\bmeta) \nonumber
\end{align}
where $\underline{\mathtt{f}}$ and $\overline{\mathtt{f}}$ are non-increasing function of $\bmeta$.
By the Karush–Kuhn–Tucker (KKT) conditions, the solution of \eqref{eq:supp_qp} is as follows:
\begin{align}
\label{eq:supp_soln}
    f(\hat{\bmf}) = \left[\hat{\bmf} + \lambda^* \mathbbm{1} \right]_{\underline{\mathtt{f}}}^{\overline{\mathtt{f}}}
\end{align}
where $\left[ \cdot \right]_{\underline{\mathtt{f}}}^{\overline{\mathtt{f}}}$ stands for lower and upper clipping by $\underline{\mathtt{f}}$ and $\overline{\mathtt{f}}$, and $\lambda^*$ is the Lagrangian multiplier. We find $\lambda^*$ such that $\mathbbm{1}^\top f(\hat{\bmf}) = b$ using binary search. Since $f(\hat{\bmf})$ is clipped by $\underline{\mathtt{f}}$ and $\overline{\mathtt{f}}$, the search range of $\lambda^*$ is $[\min_i (\hat{\bmf}_i - \underline{\mathtt{f}}_i), \max_i (\overline{\mathtt{f}}_i - \hat{\bmf}_i) ]$, where $\hat{\bmf}_i$ stands for the $i$th element in $\hat{\bmf}$, and $\underline{\mathtt{f}}_i$, $\overline{\mathtt{f}}_i$ stand for the $i$th element in $\underline{\mathtt{f}}(\bmeta)$ and $\overline{\mathtt{f}}(\bmeta)$ respectively.
Here we consider a general constraint where there are both lower and upper bounds on $\mathtt{f}$. If there is only a lower bound constraint on $\mathtt{f}$ as in \cref{eq:cbfqp_c}, we search $\lambda^*$ in $[\min_i (\hat{\bmf}_i - \underline{\mathtt{f}}_i), -\min_i \hat{\bmf}_i]$, because if $\lambda^* > -\min_i \hat{\bmf}_i$, then $\mathbbm{1}^\top \mathtt{f} > 0$, violating \cref{eq:cbfqp_b}.

To differentiate through the solver in training, we derive the derivatives based on the binding conditions of the inequality constraints. First, we define the binding and not binding sets as follows:
\begin{align}
    & \mathcal{S} = \{ i | f_i = \underline{\mathtt{f}}_i \text{ or } f_i = \overline{\mathtt{f}}_i \}, \quad \mathcal{S}^c = \Omega \backslash S \\
    & \mathcal{S}_{l} = \{ i | f_i = \underline{\mathtt{f}}_i \}, \quad \mathcal{S}_{l}^c = \Omega \backslash \mathcal{S}_{l} \\
    & \mathcal{S}_{u} = \{ i | f_i = \overline{\mathtt{f}}_i \}, \quad \mathcal{S}_{u}^c = \Omega \backslash \mathcal{S}_{u} 
\end{align}
where $\Omega = \{ i \in \mathbb{Z}^+ | i \leq n \}$. Then the derivatives of $f$ with respect to the inputs $\hat{\bmf}$, $\underline{\mathtt{f}}$ and $\overline{\mathtt{f}}$ are as follows:
\begin{equation} 
\frac{d f_i}{d \hat{f}_j}=\left\{
\begin{array}{lll}
     0, & i \in \mathcal{S}^c \text{ or } j \in \mathcal{S}^c \\
     1 - \frac{1}{n(\mathcal{S})}, & i=j\in \mathcal{S} \\
     - \frac{1}{n(\mathcal{S})} & i \neq j, i \in \mathcal{S}, j \in \mathcal{S}
\end{array}
\right.
\end{equation}
\begin{equation}
    \frac{d f_i}{d \underline{\mathtt{f}}_j}=\left\{
\begin{array}{lll}
     0, & j \in \mathcal{S}_{l}, \forall i \in \Omega \\
     0, & j \in \mathcal{S}_{l}^c, i \in \mathcal{S}_{l}^c \backslash \{ j \} \\
     1, & j \in \mathcal{S}_{l}^c, i=j \\
     - \frac{1}{n(\mathcal{S}_{l})} & j \in \mathcal{S}_{l}^c, i \in \mathcal{S}_{l}
\end{array}
\right. \quad
    \frac{d f_i}{d \overline{\mathtt{f}}_j}=\left\{
\begin{array}{lll}
     0, & j \in \mathcal{S}_{u}, \forall i \in \Omega \\
     0, & j \in \mathcal{S}_{u}^c, i \in \mathcal{S}_{u}^c \backslash \{ j \} \\
     1, & j \in \mathcal{S}_{u}^c, i=j \\
     - \frac{1}{n(\mathcal{S}_{u})} & j \in \mathcal{S}_{u}^c, i \in \mathcal{S}_{u}
\end{array}
\right.
\end{equation}

\subsection{Interval Bound Propagation through CBF-QP}
\label{sec: appendix_ibp_qp}
The dynamics of our NODE is parameterized by a neural network followed by a CBF-QP layer. 
Let $\hat{\bmf}(\bmeta)$ be the dynamics output by the neural network, and let $f(\hat{\bmf})$ be the dynamics after CBF-QP layer. 
Given perturbed input in an interval bound $\underline{\bmeta_i} \leq \bmeta_i \leq \overline{\bmeta_i}$, we first use a a popular linear relaxation based verifier named CROWN \cite{zhang2018efficient} to get an interval bound for $\hat{\bmf}$: $\underline{\hat{\bmf}_i} \leq \hat{\bmf}_i \leq \overline{\hat{\bmf}_i}$.
However, CROWN does not support perturbation analysis on differentiable optimization layers such as our CBF-QP layer and deriving linear relaxation for CBF-QP can be hard.
However, it is possible to derive interval bounds (a special case of linear bounds in CROWN) through CBF-QP.
Consider a QP in the form of \Cref{eq:cbfqp_a,eq:cbfqp_b,eq:cbfqp_c},
we bound each dimension of $f(\hat{\dyn})$ in $\mathcal{O}(n)$ by solving the QP with the corresponding element of the input set to the lower or upper bound (\Cref{prop: ibp-cbf-qp}).

\begin{proposition}
\label{prop: ibp-cbf-qp}
Consider a CBF-QP in the form of \ref{eq:supp_qp}. Define function $h_i$ to be $h_i: \bmeta, \hat{\bmf} \mapsto f(\hat{\bmf})_i $. Given perturbed input in an interval bound $\underline{\bmeta_i} \leq \bmeta_i \leq \overline{\bmeta_i}$, and $\underline{\hat{\bmf}_i} \leq \hat{\bmf}_i \leq \overline{\hat{\bmf}_i}$, we have
\begin{align}
\label{eq:supp_ibp_prop}
    h_i(\bmeta_{ub}^i, \hat{\bmf}_{lb}^i) \leq f(\hat{\bmf})_i \leq h_i(\bmeta_{lb}^i, \hat{\bmf}_{ub}^i)
\end{align}
where $\bmeta_{ub}^i, \bmeta_{lb}^i$ and $\hat{\bmf}_{ub}^i, \hat{\bmf}_{lb}^i$ are defined as follows:
\begin{equation} 
\bmeta_{ub}^i=\left\{
\begin{array}{lll}
     \overline{\bmeta_j}, & j=i \\
     \underline{\bmeta_j}, & j \neq i
\end{array}
\right. \quad
\bmeta_{lb}^i=\left\{
\begin{array}{lll}
     \underline{\bmeta_j}, & j=i \\
     \overline{\bmeta_j}, & j \neq i
\end{array}
\right.
\end{equation}

\begin{equation} 
\hat{\bmf}_{ub}^i=\left\{
\begin{array}{lll}
     \overline{\hat{\bmf}_j}, & j=i \\
     \underline{\hat{\bmf}_j}, & j \neq i
\end{array}
\right. \quad
\hat{\bmf}_{lb}^i=\left\{
\begin{array}{lll}
     \underline{\hat{\bmf}_j}, & j=i \\
     \overline{\hat{\bmf}_j}, & j \neq i
\end{array}
\right.
\end{equation}

\end{proposition}

\begin{proof}
We prove by contradiction. For the upper bound $f(\hat{\bmf})_i \leq h_i(\bmeta_{lb}^i, \hat{\bmf}_{ub}^i)$, suppose $f(\hat{\bmf})_i > h_i(\bmeta_{lb}^i, \hat{\bmf}_{ub}^i)$. Plug in \Cref{eq:supp_soln}, we have
\begin{equation}
\label{eq:supp_ibp_i}
    \left[ \hat{\bmf}_i + \lambda \right]_{\underline{\mathtt{f}}(\bmeta_i)}^{\overline{\mathtt{f}}(\bmeta_i)} > \left[ \overline{\hat{\bmf}_i} + \lambda' \right]_{\underline{\mathtt{f}}(\overline{\bmeta_i})}^{\overline{\mathtt{f}}(\underline{\bmeta_i})}
\end{equation}
Since $\underline{\mathtt{f}}$ and $\overline{\mathtt{f}}$ are non-increasing function of $\bmeta$, we have $\underline{\mathtt{f}}(\overline{\bmeta_i}) \geq \underline{\mathtt{f}}(\bmeta_i)$ and $\overline{\mathtt{f}}(\underline{\bmeta_i}) \geq \overline{\mathtt{f}}(\bmeta_i)$. Then $\hat{\bmf}_i + \lambda > \overline{\hat{\bmf}_i} + \lambda'$. Since $\hat{\bmf}_i \leq \overline{\hat{\bmf}_i}$, we have $\lambda > \lambda'$.
Then for all $j \neq i$, we have 
\begin{equation}
\label{eq:supp_ibp_j}
    \left[ \hat{\bmf}_j + \lambda \right]_{\underline{\mathtt{f}}(\bmeta_j)}^{\overline{\mathtt{f}}(\bmeta_j)} \geq \left[ \underline{\hat{\bmf}_j} + \lambda' \right]_{\underline{\mathtt{f}}(\overline{\bmeta_i})}^{\overline{\mathtt{f}}(\underline{\bmeta_i})}    
\end{equation}
Sum on both sides of \ref{eq:supp_ibp_i} and \ref{eq:supp_ibp_j}, we have 
\begin{equation}
    \mathbbm{1}^\top f(\bmeta, \hat{\bmf}) > \mathbbm{1}^\top f(\bmeta_{lb}, \hat{\bmf}_{ub})
\end{equation}
which is contradictory to the equality constraint in \ref{eq:supp_qp}. Therefore, we have $f(\hat{\bmf})_i \leq h_i(\bmeta_{lb}^i, \hat{\bmf}_{ub}^i)$. The lower bound in \ref{eq:supp_ibp_prop} can be proved in the same way.
\end{proof}

\section{Sampling Algorithms for Certification}

\begin{algorithm2e}
\SetAlgoLined
\SetAlCapHSkip{0.0em}
\SetKwInOut{Input}{Input}
\Input{Number of classes $K$, sample density $T$, solution set $\operatorname{sol}$ with dimension $T \times K$.}
      // \emph{Initialize $\operatorname{sol}$.} \\
      Initialize each element of $\operatorname{sol}$ to be $\emptyset$. \\
      $\operatorname{sol}[0][k]=\{\mathbf{0}_k\}$, where $\mathbf{0}_k=[0,..,0] \in \mathbbm{R}^k$. \\
      $\operatorname{sol}[j][2]=\{[j/2, j/2]\}$.  \\
      // \emph{Append elements to $\operatorname{sol}[j][k]$.} \\
      \For {$j$ from 2 to $T$}{
        \For {$k$ from 3 to $K$}{
            \For {$l$ from 0 to $k-2$} { 
            \If{$j-k+l \geq 0$ and $k-l \geq 0$}{
                Let $\mathcal{C}$ be the set that contains $k-l-1$ combinations of $\{1,2,..., k-1 \}$.\\
                $\operatorname{sol}[j][k] = \operatorname{sol}[j][k] \cup \{ G(y,a,c,k) | a \in \operatorname{sol}[j-k+l][k-l], c \in \mathcal{C} \}$.
            }
            }
        }
      }
 \Return{$\tilde{S}_y = \operatorname{sol}[T][K] / T$}
 \caption{Sample the points on the decision boundary by dynamic programming.}
 \label{alg:dp_decisionboundary}
\end{algorithm2e}

We describe the process of generating samples on the decision boundary in Algorithm \ref{alg:dp_decisionboundary}. The trick is to break down the $n$-class decision boundary sampling problem to 2 to $(n-1)$-class sampling problems. For instance, to generate samples with $k$ non-zero elements on an $n$-class decision boundary with density $T$, one can sample points on an $k$-class decision boundary with density $T-k$ first, adding 1 to each dimension to make each element non-zero, and assign each element to an $n$ dimensional vector.   
This operation is denoted by function $G$ in Algorithm \ref{alg:dp_decisionboundary}. $G$ takes two list inputs $a$ and $c$, increases each element in $a$ by 1, rearranges the elements in $a$ according to the indices given by $c$, and output a new list $w$ of shape $k$. Equation \ref{eqn:sample_boundary_rearrange} gives the form of the output of $G$. If the inputs are $y=0, a=(3,2,3,0), c=\{2,3,7\}$ and $k=8$ ($y$ is the label, $a$ corresponds to a point on 4-class decision boundary, and $c$ specifies the non-zero dimension except for the label dimension in an 8 dimensional vector), then the output is $w = (4,0,3,4,0,0,0,1)$.
\begin{equation} \label{eqn:sample_boundary_rearrange}
w_i=\left\{
\begin{array}{lll}
     a_0+1, & i=y \\
     a_{|\{1,2,\cdots,i\}\cap c|}+1, & i\in c  \\
     0. & \text{o/w} 
\end{array}
\right.
\end{equation}

\FloatBarrier

\section{Experiment Details}
\subsection{Nonlinear control}
\label{sec:appen_segway}
\paragraph{Baselines.}
Although there are many baselines in robust control/RL, the settings and goals are not the same as our work (certified robust forward invariance) and thus not directly comparable (Table \ref{tab:baseline-setting}). We found the setting and goal in [3] is the most similar to us, and thus we compare with their method in Table \ref{tab:robust_control_results}.

\begin{table}[h!]
\centering
\resizebox{\columnwidth}{!}{
\begin{tabular}{llll}
\toprule
Paper & Setting & Goal & Certified or not \\ 
\midrule
Robust MPO \cite{mankowitz2019robust} & \begin{tabular}[c]{@{}l@{}}Continuous MDP. \\ model mis-specification\end{tabular} & \begin{tabular}[c]{@{}l@{}}Maximize worst case \\ RL performance.\end{tabular} & No \\
\midrule
CROP \cite{wu2021crop} & \begin{tabular}[c]{@{}l@{}}Discrete MDP. \\ Input state perturbations\end{tabular} & \begin{tabular}[c]{@{}l@{}}Certification of per-state \\ actions and lower bound of \\ cumulative rewards\end{tabular} & \begin{tabular}[c]{@{}l@{}}Yes for action and \\ cumulative reward\end{tabular} \\
\midrule
Robust MBP \cite{donti2020enforcing} & \begin{tabular}[c]{@{}l@{}}Norm-bounded linear \\ differential inclusions. \\ Disturbance is norm-bounded \\ and added to the dynamics\end{tabular} & \begin{tabular}[c]{@{}l@{}}Train a nonlinear controller \\ that satisfies the exponential \\ stability condition under \\ perturbed dynamics\end{tabular} & Yes for stability \\
\midrule
Robust FI-ODE (ours) & \begin{tabular}[c]{@{}l@{}}Continuous nonlinear \\ dynamics. Norm-bounded \\ system parameter \\ perturbations.\end{tabular} & \begin{tabular}[c]{@{}l@{}}Train a nonlinear controller \\ that satisfies the forward \\ invariance (safety) condition \\ under perturbed dynamics\end{tabular} & \begin{tabular}[c]{@{}l@{}}Yes for forward invariance \\ (safety)\end{tabular} \\
\bottomrule
\end{tabular}
}
\caption{Settings of baseline methods.}
\label{tab:baseline-setting}
\end{table}

\FloatBarrier

\paragraph{Experiment details.}
The dynamics for the segway system is \cite{gurriet2018towards}:
\begin{equation}
   \frac{d}{d t}\left[\begin{array}{c}
    \phi \\
    v \\
    \dot{\phi}
    \end{array}\right] = 
    \left[ \begin{array}{c}
    \dot{\phi} \\
    \frac{\cos{\phi}(-1.8u+11.5v+9.8 \sin{\phi}) - 10.9u + 68.4v - 1.2 \dot{\phi}^2 \sin{\phi}}{\cos{\phi} - 24.7} \\
    \frac{(9.3u-58.8v)\cos{\phi}+38.6u-234.5v-\sin{\phi}(208.3+\dot{\phi}^2\cos{\phi})}{\cos^2{\phi}-24.7}
    \end{array} \right]
\end{equation}
All the constants except the acceleration of gravity $g=9.8$ are system parameters. We enforce robust forward invariance under $\pm 2\%$ perturbation on each of the parameters.
We use a 3 layer MLP as the controller and use the Adam optimizer \cite{kingma2014adam}. First, we train the controller to imitate a Linear Quadratic Regulator (LQR) controller. 
Then we jointly learn the Lyapunov function and the controller. 
We use adversarial training on $\ins$ and $\state$ to encourage the smoothness of $\dyn_\params(\state, \ins)$ with respect to $\ins$. Specifically, we find $\adv_x$ and $\adv_{\eta}$ that maximizes $\frac{\partial V}{\partial \state }^\top \dyn_\params(\state + \adv_{\eta}, \ins + \adv_x)$ and train on $\state + \adv_{\eta}$ and $\ins + \adv_x$.
We set $\kappa(V(\bmeta))$ in \Cref{eq:mc_lya_loss} to be a constant:  $\kappa(V(\bmeta))= \kappa' \leq  \eps L_{V} L_{f}^x$ ($\kappa'$ is smaller than the requisite lower bound since we train with adversarial inputs and the requisite lower bound is for nominal inputs).
We use learning rate of 0.02 for the Lyapunov function and 0.01 for the controller.
Next, we jointly learn the controller and finetune the lyapunov function from the previous stage via adversarial training on both the system states and system parameters. We use learning rate of 0.01 for the controller and 0.002 for the lyapunov function.

To certify forward invariance, we use rejection sampling on the state space to cover the boundary of the forward invariant set. The spacing of the ambient grid is set to 0.01 for all 3 dimensions, and the rejection criteria is \Cref{eq:def_B}. 
For \emph{robust} forward invariance, we certify in 2 phases. In phase 1, we set the spacing of the ambient grid to be 0.005, and we also sample a grid on the system parameter space to cover the $\pm 2\%$ perturbation range on the parameters with the spacing in \Cref{tab:spacing}. In phase 2, we sample denser states and system parameters around the states that cannot be certified in phase 1. We set the spacing along each state dimension to be 0.0025, and the spacing of the parameters to be those in brackets in \Cref{tab:spacing}.

\begin{table}[h]
\caption{Spacing of the sampled grid on the system parameter space in terms of percentage of each parameter value. Spacing for phase 2 is in the bracket.}
\begin{center}
\resizebox{\linewidth}{!}{
\begin{tabular}{lccccccccccc}
\toprule
Parameter & 1.8 & 11.5 & 10.9 & 68.4 & 1.2 & 9.3 & 58.8 & 38.6 & 234.5 & 208.3 & 24.7 \\ 
\hline
Spacing ($\%$) & 4 (4) & 4 (4) & 4 (4) & 1 (1) & 4 (4) & 4 (4) & 2 (1) & 1 (1) & 1 (0.25) & 4 (4) & 4 (4) \\
\bottomrule
\end{tabular}
\label{tab:spacing}
}
\end{center}
\end{table}

We run all the control experiments on Intel Core i9 CPU. The certification time for forward invariance is 3.1 seconds, while for robust forward invariance, it is 3285.3 seconds. We also report the training time and the standard deviation of forward invariant rate of each method in \Cref{tab:appen_robust_control_results}.

\begin{table}[h]
\vspace{-10pt}
  \centering
\caption{Robustness of controllers trained with different training methods. The numbers are the percentage of trajectories that stay within the forward invariant set under the nominal and adversarial system parameters on 1000 adversarially selected initial states. We report the mean and standard deviation over 3 runs. The certificate column indicates whether or not we can certify the (robust) FI property.}  
  \resizebox{0.85\linewidth}{!}{  
  \renewcommand{\arraystretch}{1.2}
  \begin{tabular}{lccccc}
    \hline
    \multirow{2}{*}{\textbf{Training Method}} & \multirow{2}{*}{\textbf{Training Time (s)}} & \multicolumn{2}{c}{\textbf{Empirical FI rate (\%)}} & \multicolumn{2}{c}{\textbf{Certificate}}  \\
    \cline{3-6}
    &  & \textbf{Nominal Params} & \textbf{Adv Params} & \textbf{FI} & \textbf{Robust FI} \\
    \hline
    Standard Backprop Training & 191.3 & 58.0 $\pm$ 1.9 & 50.4 $\pm$ 1.1 & \xmark & \xmark \\
     Basic Lyapunov Training \cite{jimenez2022lyanet} & 1.5 & 90.2 $\pm$ 0.3 & 52.6 $\pm$ 0.6 & \xmark & \xmark \\     
    \  + Adaptive Sampling & 3.3 & 100 $\pm$ 0.0 &  68.9 $\pm$ 0.3 & \cmark & \xmark \\
      \ + Adversarial Training & 67.7 & 100 $\pm$ 0.0 & 97.8 $\pm$ 0.3 & \cmark & \xmark \\       
    \textbf{\ + Both (Ours)} & 96.7 & 100 $\pm$ 0.0  & 100 $\pm$ 0.0 & \cmark & \cmark \\ 
    \hline
  \end{tabular}
  }
  \label{tab:appen_robust_control_results}
\end{table}

\subsection{Image classification}

\paragraph{Useful techniques for classification problems.} 
Typically the decision boundary $\{ \state \in \BR^n | \bmeta_y = \max_{i \neq y} \bmeta_i \}$ is not compact on the logit space ($\bmeta \in \BR^n$). However, we need the Lyapunov level set to be compact for certification because we can only sample finite points and verify the conditions hold in their neighborhoods. Therefore, we restrict the states to evolve on a probability simplex for classification problems. To do so, we use a Control-Barrier Function based Quadratic Program (CBF-QP) \cite{ames2016control}, implemented as a differentiable optimization layer \cite{cvxpylayers2019} in the dynamics (Appendix \ref{sec: simplex-arch}). However, the linear relaxation based verifier that we use (CROWN) \cite{zhang2018efficient} does not support perturbation analysis on differentiable optimization layers such as our CBF-QP layer. Since deriving linear relaxation for CBF-QP is hard, we derive an interval bound (a special case of linear bounds in CROWN) through CBF-QP (\Cref{sec: appendix_ibp_qp}). 

\paragraph{Experiment settings.}
For image classification tasks, we use orthogonal layers \cite{trockman2021orthogonalizing} in the neural network so that the dynamics has 1 Lipschitz constant with respect to both the state and the input. Specifically, we have $\hat{\bmf}(\bmeta, x) = W_3 \sigma(W_2 \sigma(W_1 \bmeta + g(x)) + b_2) + b_3$, where $g$ is a neural network with 4 orthogonal convolution layers and 3 orthogonal linear layers, and $W_1, W_2, W_3$ are orthogonal matrices, $\sigma$ is the ReLU activation function. 
We set $\kappa(V(\bmeta)) = \eps L_{V} L_{f}^x V(\bmeta)$ in the training loss (\Cref{eq:mc_lya_loss}).  Note that on the decision boundary, $V(\bmeta)=1$.

In the CBF-QP, we need to pick a class $\mathcal{K}^e_{\infty}$ function $\alpha$ for the inequality constraint $\mathtt{f} \geq -\alpha(\bmeta)$. Here we use $\alpha(\bmeta) = c_1 (e^{c_2 \bmeta} - 1)$, where $c_1=100$, and $c_2=0.02$. Comparing with a linear function, this $\alpha(\bmeta)$ leads to a higher margin over Lipschitz ratio, resulting in better certified accuracy.

During training, we train with batch size of 64. For each image, we sample 512 states. From epoch 1 to 10, all the states are uniformly sampled in the simplex. From epoch 11 to epoch 60, we linearly decay the proportion of uniform sampling in the simplex and increase the portion of uniform sampling within the correct classification set for each class. To sample uniformly in the simplex, we first sample $n$ points from exponential distribution $\text{Exp}(1)$ independently, then we normalize the $n$ dimensional vector to have sum 1. To sample uniformly in the correct classification region for each class, we first uniformly sample from the simplex, then we swap the maximum element with the element corresponding to the correct label. We choose $\kappa$ to be 2.0 in the loss function (\cref{eq:mc_lya_loss}). We use Adam optimizer with learning rate 0.01, and train for 300 epochs in total.

For certification, we choose $\den=40$ when sampling on the decision boundary. A larger $\den$ will lead to better certified accuracy but increases the computational cost dramatically.
We ran the experiments on an NVIDIA RTX A6000 GPU.

For baseline methods \cite{pabbaraju2020estimating, chen2021semialgebraic}, the adversarial accuracy is evaluated with PGD attack. For our methods, we use AutoAttack \cite{croce2020reliable} to evaluate the empirical adversarial robustness.

\begin{table}[h]
\caption{Computational costs for certification on CIFAR-10.}
\begin{center}
\resizebox{0.6\linewidth}{!}{
\begin{tabular}{lcccc}
\toprule
\thead{\textbf{Certification} \\ \textbf{Method}} & \thead{\textbf{Sampling} \\ \textbf{density ($N$)}} & \textbf{\# samples} & \textbf{Time (s)} & \textbf{Certified} \\
\hline
Lipschitz & 20 & $3.67 \times 10^5$ & 1.03 & 0 \\ 
Lipschitz & 30 & $5.50 \times 10^6$ & 1.37 & 27.40 \\
Lipschitz & 40 & $4.13 \times 10^7$  & 2.8 & 33.46 \\ 
CROWN & 40 & $4.13 \times 10^7$ & 240 & 42.27 \\
\bottomrule
\end{tabular}
}
\label{tab:abla-costs}
\end{center}
\end{table}

\paragraph{Computational cost.}
The main computational costs of our method comes from the number of samples that are needed to cover the boundary of the forward invariant set.
\Cref{tab:abla-costs} compares the computational costs and performance for different certification methods on CIFAR-10. 
We first compare the results of certifying with Lipschitz bounds and CROWN \cite{zhang2018efficient}. Certifying with Lipschitz bounds is faster. Since we can pre-compute the Lipschitz bound of the dynamics, the certification time equals to the inference time on all the states. 
Certifying with CROWN provides a tighter bound and thus higher certified accuracy but is more computationally expensive than using the Lipschitz bound.
We also compare the performance of different sampling density by choosing different $N$ in \Cref{eq:sample_simplex}. With larger $N$, we can cover the region of interest with smaller neighborhood around each sampled point. We vary $N$ using the Lipschitz certification method because it is faster to evaluate, but the pattern should remain the same for CROWN. As expected, we get better accuracy with larger sampling density, but the computational time is longer since we have more samples.

\end{document}